\documentclass{article}

\usepackage{microtype}
\usepackage{graphicx}
\usepackage{subfigure}
\usepackage{booktabs} 

\usepackage{hyperref}

\newcommand\myeq{\mathrel{\stackrel{\makebox[0pt]{\mbox{\normalfont\tiny def}}}{=}}}

\usepackage[accepted]{icml2024}

\usepackage{amsmath}
\usepackage{amssymb}
\usepackage{mathtools}
\usepackage{amsthm}

\usepackage[capitalize,noabbrev]{cleveref}

\theoremstyle{plain}
\newtheorem{theorem}{Theorem}[section]
\newtheorem{proposition}[theorem]{Proposition}
\newtheorem{lemma}[theorem]{Lemma}
\newtheorem{corollary}[theorem]{Corollary}
\theoremstyle{definition}
\newtheorem{definition}[theorem]{Definition}

\theoremstyle{remark}

\usepackage[textsize=tiny]{todonotes}

\usepackage{tikz}
\usetikzlibrary{automata, arrows.meta, positioning, quotes}
\icmltitlerunning{On the Tractability of SHAP Explanations 
           under Markovian Distributions}

\begin{document}

\twocolumn[
\icmltitle{On the Tractability of SHAP Explanations 
           under Markovian Distributions}



\icmlsetsymbol{equal}{*}

\begin{icmlauthorlist}
\icmlauthor{Reda Marzouk}{yyy}
\icmlauthor{Colin de la Higuera}{yyy}
\end{icmlauthorlist}

\icmlaffiliation{yyy}{LS2N, Universit\'e de Nantes, France}

\icmlcorrespondingauthor{Reda Marzouk}{mohamed-reda.marzouk@univ-nantes.fr}
\icmlkeywords{SHAP score, Weighted Automata, Disjoint DNFs, Decision Trees, Markovian distributions, explainable ML}

\vskip 0.3in
]



\printAffiliationsAndNotice{} 

\begin{abstract}
Thanks to its solid theoretical foundation, the SHAP framework is arguably one the most widely utilized frameworks for local explainability of ML models. Despite its popularity, its exact computation is known to be very challenging, proven to be NP-Hard in various configurations. Recent works have unveiled positive complexity results regarding the computation of the SHAP score for specific model families, encompassing decision trees, random forests, and some classes of boolean circuits. Yet, all these positive results hinge on the assumption of feature independence, often simplistic in real-world scenarios. In this article, we investigate the computational complexity of the SHAP score  by relaxing this assumption and introducing a Markovian perspective. We show that, under the Markovian assumption, computing the SHAP score for the class of Weighted automata, Disjoint DNFs and Decision Trees can be performed in polynomial time, offering a first positive complexity result for the problem of SHAP score computation that transcends the limitations of the feature independence assumption.
\end{abstract}

\section{Introduction}
Since its introduction in the seminal paper \cite{lundberg2017}, the local explanatory (\textbf{SH}apley \textbf{A}dditive ex\textbf{P}lanations) SHAP method gained increasing popularity in the field of interpretable ML. Nevertheless, one of its main limitations pertains to its computational intractabilty: In general, computing the SHAP score is NP-Hard \cite{bertossi2020causality, vander21}.  
Recent studies have shown positive results regarding the tractability of computing the SHAP score under specific configurations. In particular, \cite{lundbergnature} proposed a polynomial-time algorithm, known as TreeSHAP, that purported to compute exactly the SHAP score for tree-based models. However, subsequent research \cite{vander21, arenas23} has identified flaws in the main claim of TreeSHAP. Indeed, other works have demonstrated that the TreeSHAP algorithm is an implementation of \textit{interventional} SHAP, another variant of the SHAP score \cite{janzing20a}. 
A more rigorous proof for the tractability of the original SHAP score for various families of boolean functions has been provided in \cite{vander21} when features are assumed to be independent. \cite{arenas23} extended these positive results to cover the family of Decomposable Deterministic circuits which includes the family of decision trees among other classes of boolean circuits.

All the tractability results reported in the literature are, however, derived under the feature independence assumption. Although practical for its simplicity, this assumption is often irrealistic in real-case scenarios. A slight relaxation of this assumption has been examined in \cite{vander21} through the lens of complexity theory by considering the family of na\"ive bayes models and empirical distributions. Computing the SHAP score for the family of decision trees under this relaxed assumption has been proven to be \#P-Hard in both these settings.  

Between independent distributions and latent variable models, an intermediate class of distributions that hasn't been explored yet is the class of Markovian distributions.  Markovian distributions constitute an interesting class of distributions that incorporate a degree of feature correlation often considered sufficient to model various stochastic phenomena \cite{bassler2006markov, vanKampen2007, goutsias13}.  

Previous works examining the complexity of computing the SHAP score were mostly directed towards families of boolean functions. In this article, we shift our focus to sequential models, in particular the family of weighted automata (WAs). WAs offer a powerful formalism for modeling sequential tasks and encompass a large family of classical models, including Deterministic and Non deterministic finite automata, Hidden Markov Models, and has been shown to be equivalent to second-order linear RNNs \cite{rabusseau19}. They have been employed in various applications, such as NLP \cite{knight2009applications}, speech processing \cite{pereira1996speech, mohri2008speech} and image processing \cite{CULIK1993305}. 

Recently, a line of works proposed WAs as \textit{proxy} interpretation models for neural models \cite{okoduno19, eyraud20, weiss, rabusseau}. All these works are motivated by the implicit assumption that WAs enjoy better transparency than their neural counterparts. However, the existing litterature lacks a formal argument to substantiate this claim. One of the primary motivations of this work is to shed some light on this issue. 

The work presented in this article will primarily address the original formulation of the SHAP score, as introduced by \cite{lundberg2017} in their seminal paper. This specific variant of the SHAP score has proven to be particularly challenging from a computational viewpoint, and extends, up to a distributional shift, other variants such as the baseline SHAP \cite{sundararajan20b}. It's worth noting, however, that the axiomatic basis of the original SHAP score has been recurrently disputed in the academic discourse with several works discussing its limitation to capture elementary desirable properties of local models' explanations \cite{janzing20a, sundararajan20b,Huang2023TheIO}. 

The main results presented in this article are given as follows:
\vspace{-0.2cm}
\begin{enumerate}
\item A constructive proof showing that the computation of the SHAP score for the class of WAs is tractable under the assumption that the background data generating distribution is Markovian (section \ref{results:ffdist}).
 \item Under the same assumption, a constructive proof of the tractability of computing the SHAP score for the class of disjoint DNFs and the family of decision trees (section \ref{results:reduction}).
\end{enumerate}
 \section{Background} \label{background}
 For a given integer $n > 0$, we denote by $[n]$ the set of all integers from $1$ to $n$. The indicator function of a set $X$ shall be denoted as $I_{X}$. Recall that an indicator function of a subset $X$ in $\mathcal{X}$ is a binary-valued function that assigns the value 1 to $x \in X$, 0 otherwise. 

 A computational function problem $f: \mathcal{I} \rightarrow \mathbb{R}$, where $\mathcal{I}$ is referred to as the set of instances, is in FP if it can be computed exactly using an algorithm that runs in time polynomial in the size of the instance. 
 
$\bullet$ \textbf{Languages and seq2seq languages.} 
Let $\Sigma$ be a finite alphabet. The elements of $\Sigma$ will be referred to as symbols. $\Sigma^{*}$ (resp. $\Sigma^{\infty}$) denotes the set of all finite (resp. infinite) sequences formed by $\Sigma$. For a given sequence $w \in \Sigma^{*}$, we denote by $|w|$ its length, $w_{i:j}$ the subsequence of $w$ that spans from the $i$-th symbol to the $j$-th symbol in $w$, and $w_{i}$ to refer to its i-th symbol. A language $f$ is a mapping from $\Sigma^{*}$ to $\mathbb{R}$. When the image of a language $f$ is binary, then it will be called \textit{unweighted}, in which case the language represents a subset of $\Sigma^{*}$ equal to $L_{f} = f^{-1}(\{1\})$. We extend the definition of languages to cover unweighted languages over $\Sigma^{*}$, by allowing the notation: $f(L) \myeq \sum\limits_{w \in L} f(w)$ (if it exists) for an unweighted language $L$. An analogous concept of a language is the concept of a \textit{seq2seq language}. For two finite alphabets $\Sigma$ and $\Delta$, a seq2seq language is a mapping from $\Sigma^{*} \times \Delta^{*}$ to $\mathbb{R}$. 

 When a language (or, a seq2seq language) $f$ is computed by a model  $M$, such as a weighted automaton (WA) or a weighted transducer (WT), we shall use the notation $f_{M}$ to designate the language (or, seq2seq language) computed by $M$. 

$\bullet$ \textbf{Operators over languages/seq2seq languages.} In this article, three operators over languages will be useful in our analysis. We shall briefly define them in the following:
\begin{enumerate}
    \item \textbf{The product operator:} The product operator, also known as \textit{the hadamard product \cite{Droste2009, Mohri2004}}, takes two languages $f,~g$ over $\Sigma^{*}$ and outputs the product language $f \cdot g$. We shall employ the notation $f \otimes g$ to refer to the product language of $f$ and $g$.
    \item \textbf{The partition constant operator:} The partition constant operator takes a language $f$ over $\Sigma^{*}$, an integer $n$, and outputs the quantity $f(\Sigma^{n}) = \sum\limits_{w \in \Sigma^{n}} f(w)$. The partition constant operation of a language $f$ at the support $n>0$ will be denoted as $|f|_{n}$.
    \item \textbf{The projection operator:} The projection operator takes as input a language $f$ over $\Sigma^{*}$ and a seq2seq language $g$ over $\Sigma^{*} \times \Delta^{*}$ and outputs a language $h$ over $\Delta^{*}$ given as
    $$h(u) = \sum\limits_{w \in \Sigma^{|u|}} f(w) \cdot g(w,u)$$
    In the sequel, we shall use the notation $\Pi(f,g)$ to refer to the projection operator.
\end{enumerate}

$\bullet$ \textbf{Patterns.} For an alphabet $\Sigma$, a pattern $p$ is a regular expression that takes the form: $\Sigma^{i_{1}}w_{1} \ldots \Sigma_{i_{n}}w_{n}\Sigma^{i_{n+1}}$, where $\{i_{k}\}_{k \in [n+1]}$ is a set of integers, and $\{w_{k}\}_{k \in [n+1]}$ is a collection of sequences over $\Sigma^{*}$. The language accepted by a pattern $p$ shall be denoted $L_{p}$. 
Analogous to sequences, the symbol $|p|$ will refer to its length. In addition, $|p|_{\#}$ will denote the number of occurrences of the symbol $\Sigma$ in $p$. 

In this article, the pattern formalism will be employed to represent coalitions of features in the SHAP score formula. Often, they shall be treated as sequences formed by an extended alphabet $\Sigma_{\#} = \Sigma \cup \{\#\}$, where $\#$ is a special symbol that replaces the symbol $\Sigma$ present in the regular expression associated to a pattern $p$. For example, the pattern $p = \Sigma 00 \Sigma^{2}$ over the binary alphabet $\Sigma = \{0,1\}$ is represented by the sequence $p = \# 00 \# \#$ over $\Sigma_{\#}$.

By treating patterns as sequences over $\Sigma_{\#}^{*}$, we can describe languages over patterns in the usual way. In particular, the following languages over patterns will be used in the remainder of this article. Given a sequence $w \in \Sigma^{*}$ and an integer $k \in [|w|]$, define the (unweighted) language over $\Sigma_{\#}^{*}$ as:
$$\mathcal{L}_{k}^{w} \myeq \{p \in \Sigma_{\#}^{|w|}:  ~~ w \in L_{p} \land |p|_{\#} = k
\}$$
The uniform distribution over the set $\mathcal{L}_{k}^{w}$ will be referred to as $\mathcal{P}_{k}^{w}$, and the language $\bigcup\limits_{k \in [|w|]} \mathcal{L}_{k}^{w}$ as $\mathcal{L}^{w}$.

A final operation over patterns that will appear in the reformulation of the SHAP score formula introduced later in this section is the $\texttt{swap}$ operation. Given a pattern $p \in \Sigma_{\#}^{*}$, an integer $i \in [|w|]$, $\texttt{swap}(p,i)$ refers to the (perturbed) pattern $p'$ generated by replacing the $i$-th element of $p$ with $\#$. For example, $\texttt{swap}(\#\#00\#1, 3) = \#\#\#0 \#1$.   

$\bullet$ \textbf{Markovian distributions.}
    Formally, a Markovian probability distribution $P$ over $\Sigma^{\infty}$ is parametrized by $<P_{init}, \{P_{n}\}_{n > 0}>$, where $P_{init}$ is a probability distribution over $\Sigma$, and for any integer $n > 0$, $P_{n}$ is a stochastic matrix \footnote{A stochastic matrix is a positive matrix such that the sum of its row elements is equal to 1}  in $|\Sigma| \times |\Sigma|$. For an integer $n > 0$, $P$ induces a probability distribution over the support $\Sigma^{n}$, denoted $P^{(n)}$, such that for any sequence $w \in \Sigma^{n}$:
    $$P^{(n)}(w) \myeq P(w \Sigma^{\infty}) = P_{init}(w_{1}) \cdot \prod\limits_{i=1}^{n-1} P_{i}[w_{i},w_{i+1}]$$
    We shall abuse notation and use $P_{i}(\sigma'|\sigma)$ instead of $P_{i}[\sigma,\sigma']$ interpreted as the probability of generating the symbol $\sigma'$ at position $i+1$ conditioned on the generation of the symbol $\sigma$ at position $i$. When there is no confusion of the support of the distribution, we shall omit the subscript from the notation $P^{(n)}$.
    
 For computational considerations, we constrain the family of Markovian distributions to those whose set of parameters can be efficiently queried:
\begin{definition} \label{assump}
    A Markovian distribution over $\Sigma^{\infty}$ is polynomial-time computable if there exists an algorithmic  procedure that takes as input an integer $n>0$, runs in $O(\texttt{poly}(n,|\Sigma|))$ and outputs the transition matrix $P_{n}$.
\end{definition}

As a notable example, the probability distribution generated by the class of 1-gram models is trivially polynomial-time computable. The set of polynomial-time computable Markovian distributions will be denoted $\texttt{MARKOV}$. When $P \in \texttt{MARKOV}$ is given as an input instance to a computational function problem, it refers to a machine that implements the algorithmic procedure defined implicitly in  definition \ref{assump}.

An additional technical assumption on Markovian distributions considered in this article is that all elements of their stochastic matrices and $P_{init}$ are greater than $0$.

\subsection{Weighted Automata/Transducers} \label{background:wa-wt}
$\bullet$ \textbf{Weighted Automata.} Weighted Automata (WAs) extend the classical family of finite automata accepting unweighted languages by allowing transitions to be endowed with weights, construed as probabilities, costs, or scores depending on the application at hand. A linear representation of WAs is formally defined as follows:
\begin{definition}{(\cite{franc})}
    Let $\Sigma$ be an alphabet and $n>0$ be an integer. A WA $A$ over $\Sigma^{*}$ is represented by a tuple $<\alpha, \{A_{\sigma}\}_{\sigma \in \Sigma}, \beta>$ where $A_{\sigma} \in \mathbb{R}^{n \times n}$ is the transition matrix associated to a symbol $\sigma$ in $\Sigma$, and $\alpha$ (resp. $\beta$) are vectors in $\mathbb{R}^{n}$ that represent the initial (resp. final) vectors. The integer $n$ is called the size of $A$, denoted $\texttt{size}(A)$.
\end{definition}
A WA $A = <\alpha, \{A_{\sigma}\}_{\sigma \in \Sigma}, \beta>$ over $\Sigma^{*}$ computes the language 
$$f_{A}(w) = \alpha^{T} \cdot A_{w} \cdot \beta$$
where $A_{w} \myeq \prod\limits_{i=1}^{|w|}A_{w_{i}}$.

$\bullet$ \textbf{Weighted transducers.} Weighted transducers (WTs) represent the analogous version of WAs adapted to model seq2seq languages. It has been employed in applications including speech processing \cite{mohri2008speech,lehr2010learning}, machine translation \cite{kumar2006weighted} and image processing \cite{culik1995weighted}

Analogous to WAs, WTs admit a linear representation given as follows:
 \begin{definition}
  Let $\Sigma,~\Delta$ be two finite alphabets and $n>0$ be an integer. A WT $T$ over $\Sigma^{*} \times \Delta^{*}$ is represented by the tuple $<\alpha, \{A_{\sigma}^{\sigma'}\}_{(\sigma,\sigma') \in \Sigma \times \Delta}, \beta>$, where $\alpha \in \mathbb{R}^{n}, A_{\sigma}^{\sigma'} \in \mathbb{R}^{n \times n},~\beta \in \mathbb{R}$. 
  The integer $n$ is called the size of $T$, denoted $\texttt{size}(T)$. 
\end{definition}
 A WT $T = <\alpha, \{A_{\sigma}^{\sigma'}\}_{\sigma \in \Sigma, \sigma' \in \Delta}, \beta>$ over $\Sigma^{*} \times \Delta^{*}$ computes the seq2seq language
$$f_{T}(w,u) = \alpha^{T} \cdot \prod\limits_{i=1}^{|w|} A_{w_{i}}^{u_{i}} \cdot \beta$$
where $(w,u) \in \Sigma^{*} \times \Delta^{*}$ such that $|w| = |u|$.

 Earlier in this section, we introduced three operators over languages/seq2seq languages, namely the product operator, the partition constant and the projection operator. The algorithmic construction we shall furnish in later sections to compute the SHAP score will involve performing a sequence of these operations over languages/seq2seq languages described by WAs/WTs whose parametrization will depend on the input instance of the problem. 

The following provides a technical lemma proving the computational efficiency of implementing these operators over languages/seq2seq languages represented by WAs/WTs. 
\begin{lemma} \label{operators} 
    Fix two finite alphabets $\Sigma,~\Delta$.
    \begin{enumerate}
        \item \textbf{The product operator.} There exists an algorithm that takes as input two WAs $A,~B$, runs in $O(\texttt{poly}(\texttt{size}(A), \texttt{size}(B), |\Sigma|))$ and outputs a WA $A \otimes B$ that computes the product language $f_{A} \otimes f_{B}$.
        \item \textbf{The partition constant operator.} There exists an algorithm that takes as input a WA $A$ and an integer $n>0$, runs in $O(\texttt{poly}(\texttt{size}(A), n, |\Sigma|))$ and outputs $|f_{A}|_{n}$.
        \item \textbf{The projection operator.} There exists an algorithm that takes as input a WA $A$, a WT $T$, runs in $O(\texttt{poly}(\texttt{size}(A), \texttt{size}(T), |\Sigma|)$ and outputs the language $\Pi(f_{A},f_{T})$.
    \end{enumerate}
\end{lemma}
The proof of lemma \ref{operators} can be found in appendix \ref{app:operators}.

\subsection{The SHAP score.} \label{background:shapscore}
 Stemming its root from the field of cooperative game theory \cite{deng1994complexity}, the SHAP framework is built on top of an analogy between cooperative games and the local explainability problem of ML models. A cooperative game is described by a set of players $N$ and a value function $v$ that assigns a \textit{generated wealth} for each subset of players, referred to as a \textit{coalition},  cooperating in the game. By analogy, in the context of explainable ML, the players are the input features of a ML model subject to explanatory analysis. And, the value assigned to a coalition is equal to the the expected model's output conditioned on the event that the features forming the coalition possess a value equal to the instance to explain. 

Similar to Shapley's original cooperative game theory \cite{shapley1953value}, the SHAP explainability method offers at its core a formal characterization of a fair distribution mechanism across input features that reflects their respective degree of contribution to the generated model's output for a given instance to explain, culminating in what's commonly known as the SHAP score.

Formally, let $M$ be a model that computes a function $f_{M}$ from a discrete set $\mathcal{X} = \mathcal{X}_{1} \times \ldots \times \mathcal{X}_{n}$ to $\mathbb{R}$, and $P$ be a probability distribution over $\mathcal{X}$. For an input $x \in \mathcal{X}$, and an integer $i \in [n]$, the original SHAP score assigned to the $i$-th feature for the instance $x$ is given as \cite{lundberg2017}:
\vspace{-0.1cm}
\begin{small}
\begin{align} \label{shapformula}
    \text{SHAP}(M,x,i,P) \myeq &  \sum\limits_{S \subseteq [n]} \frac{|S|! (n - |S| - 1)!}{n!} \cdot  \\ \nonumber
    &\left[ v(S; M,x,P)- v(S \setminus \{i\}; M,x,P) \right]
\end{align}
\end{small}
where for a subset $S \subseteq [n]$, the value function $v$ is defined as 
\begin{equation} \label{valuefunction}
    v(S;M,x,P) \myeq \mathbb{E}_{X \sim P}[f_{M}(X) | X_{S} = x_{S}]
\end{equation}

We propose an alternative formulation of the SHAP score formula tailored to better suit sequential models computing languages. For an alphabet $\Sigma$, a model $M$ that computes a language over $\Sigma^{*}$, a probability distribution $P$ over $\Sigma^{\infty}$, a string $w \in \Sigma^{*}$ and an integer $i \in [|w|]$. The SHAP value assigned to the symbol $w_{i}$ in $w$ is given as: 
\begin{small}
\begin{align} \label{shapformulaseq}
    \text{SHAP}(M,w,i,P) = \sum\limits_{k=1}^{|w|-1}  &\frac{1}{|w| - k}  \mathbb{E}_{p \sim \mathcal{P}_{|w| - k}^{w}} [V(p; M,w,P) \\ \nonumber
   & - V(\texttt{swap}(p,i); M,w,P)] 
\end{align}
\end{small}
where 
\begin{equation} \label{value}
    V(p; M,w,P)  \myeq \mathbb{E}_{w' \sim P^{|w|}} \left[f_{M}(w')| w' \in L_{p} \right]
\end{equation}
The main idea behind this reformulation consists at modeling coalitions as patterns. For example, for a sequence $w = abbaa$ over the alphabet $\Sigma=\{a,b\}$, and $k=2$, the pattern $p = \#b\#a\# $ in $\mathcal{L}_{3}^{w}$ coincides with the coalition of size $2$ formed by the second and the forth symbol of $w$. 

The faithfulness of the SHAP value formula given in \eqref{shapformulaseq} to the one in \eqref{shapformula} (for the case of sequential models) can be checked
by decomposing the summation in the original formulation of the SHAP score (equation \eqref{shapformula}) over coalitions of the same size, and by noting that for $k \in [|w| - 1]$, and a pattern $p \in \mathcal{L}_{|w| - k}^{w}$, we have
$$\mathcal{P}_{|w| - k}^{w}(p) = \frac{1}{|L_{|w| - k}^{w}|} = \frac{k! \cdot (|w| - k)!}{|w|!} $$

In the sequel, whenever the $\text{SHAP}$ score formula is mentioned, it shall refer to the one tailored for sequential models using the pattern formalism (equation \eqref{shapformulaseq}). To avoid confusion between models computing languages and boolean functions treated in section \ref{ddnfstowa}, we shall use the notation $\overrightarrow{\text{SHAP}}$ for this latter case.

The formal definition of the \textit{meta-}computational problem associated to SHAP score is given as follows:

Fix an alphabet $\Sigma$. Let $\mathcal{M}$ be a class of sequential models that compute languages over $\Sigma^{*}$, and $\mathcal{P}$ is a class of probability distributions over $\Sigma^{\infty}$. The computational \textit{meta-}problem associated to the SHAP score is given formally as follows: 

$\bullet$ \textbf{Problem:} $\texttt{SHAP}(\mathcal{M}, \mathcal{P})$\\
$~~$\textbf{Instance:} $M \in \mathcal{M}$, a sequence  $w \in \Sigma^{*}$, an integer $i \in [|w|]$, and $P \in \mathcal{P}$ \\ 
$~~$\textbf{Output:} Compute $\text{SHAP}(M,w,i,P)$

The next section is dedicated to the examination of the computational complexity of the particular instance of this problem where $\mathcal{M} = \texttt{WA}$ and $\mathcal{P} = \texttt{MARKOV}$, namely $\texttt{SHAP}(\texttt{WA}, \texttt{MARKOV})$. 

 \section{The problem $\texttt{SHAP}(\texttt{WA}, \texttt{MARKOV})$ is in \emph{FP}.} \label{results:ffdist}
The main result of the article is stated in the following theorem:
\begin{theorem} \label{ffdist:mainthm}
    The computational problem $\texttt{SHAP}(\texttt{WA}, \texttt{MARKOV})$  is in \emph{FP}.
\end{theorem}
In essence, Theorem \ref{ffdist:mainthm} states the existence of an algorithm that computes exactly the SHAP score for the class of WAs under Markovian distributions in $O(\texttt{poly}(\texttt{size}(A), |\Sigma|, |w|))$ time where $A$ is the WA given in input instance.

The remainder of this section is dedicated to provide the high-level steps of the proof of theorem \ref{ffdist:mainthm}. Technically engaged proofs of intermediary results are delegated to the appendix. At a high-level, the structure of the proof follows two steps, where the second step is decomposed in two sub-steps:

\begin{enumerate}
    \item \textit{A decomposition of the problem $\texttt{SHAP}(\texttt{WA}, \texttt{MARKOV})$:} \\
    The first step involves a decomposition of the SHAP score formula into a sum of functions, denoted $\text{SHAP}_{1},~\text{SHAP}_{2}$, which will be defined later in this section. By means of a reduction argument, we shall prove that if the computational problems associated to  $\text{SHAP}_{1},~\text{SHAP}_{2}$ are in FP, then $\texttt{SHAP}(\texttt{WA}, \texttt{MARKOV})$ is also in \emph{FP} (lemma \ref{reduction-shap}). 
    \item \textit{The problems $\text{SHAP}_{1},~\text{SHAP}_{2}$ are in \emph{FP}:} \\
    In the second step, we shall show that the computational problems associated to $\text{SHAP}_{1},~\text{SHAP}_{2}$ are in FP. (lemma \ref{shap1shap2FP}). The proof of this statement will follow two sub-steps:
    \begin{enumerate}
        \item In the first sub-step, we shall prove that the computation of $\text{SHAP}_{1}$ and $\text{SHAP}_{2}$ is reduced to performing a finite sequence of operations over languages/seq2seq languages whose parametrization will depend on the input instance of the problem (lemma \ref{ffdist:mainlemma}).
      \item In the second sub-step, we show that WAs/WTs that compute languages/seq2seq languages over which operations are performed in the previous step can be constructed.
    \end{enumerate} 
\end{enumerate}

The proof is essentially constructive, and can be translated to a practical implementation. The organisation of the remainder of this section will follow the structure of the proof given above.

\subsection{Step 1: A decomposition of the problem $\texttt{SHAP}(\texttt{WA}, \texttt{MARKOV})$.}

For a model $M$ computing a language over $\Sigma^{*}$, a sequence $w \in \Sigma^{*}$, an integer $(i,k) \in [|w|] \times [|w| - 1]$, and $P$ a probability distribution over $\Sigma^{\infty}$. 
 Define the following two functions:
\begin{equation}\label{shap1}
\begin{split}
    \text{SHAP}_{1}(M,w,k,P) \myeq  \mathbb{E}_{p \sim \mathcal{P}_{k}^{w}} V(p; M,w,P)
\end{split}
\end{equation}
and,
\begin{equation}\label{shap2}
    \text{SHAP}_{2}(M,w,i,k,P)  \myeq \mathbb{E}_{p \sim \mathcal{P}_{k}^{w}} V(\texttt{swap}(p,i); M,w,P) 
\end{equation}

By a simple manipulation of the SHAP score formula in \eqref{shapformulaseq}, we obtain  
\begin{equation}\label{ffdist:decomposition}
\begin{split}
   \text{SHAP}(M,w,i,P) =  \sum\limits_{k=1}^{|w| - 1} \frac{1}{k} & [\text{SHAP}_{1}(M,w,k,P) \\ 
   & - \text{SHAP}_{2}(M,w,i,k,P)] 
\end{split}
\end{equation}

The formal definition of the computational problems associated with the computation of $\texttt{SHAP}_{1},~\texttt{SHAP}_{2}$ for the class of WAs under the family of Markovian distributions is given as follows: 

$\bullet$ \textbf{Problem:} $\texttt{SHAP}_{1}(\texttt{WA}, \texttt{MARKOV})$ \\
$~~$\textbf{Instance:} A WA $A$, a sequence $w$ in $\Sigma^{*}$, an integer $k \in [|w|]$, $P \in \texttt{MARKOV}$ \\
$~~$\textbf{Output:} Compute $\text{SHAP}_{1}(A,w,k,P)$ 

$\bullet$ \textbf{Problem:} $\texttt{SHAP}_{2}(\texttt{WA}, \texttt{MARKOV})$ \\
$~~$\textbf{Instance:} A WA $A$, a sequence $w$ in $\Sigma^{*}$, two integers $(k,i) \in [|w|]^{2}$, $P \in \texttt{MARKOV}$ \\
$~~$\textbf{Output:} Compute $\text{SHAP}_{2}(A,w,i,k,P)$ 

The polynomial-time reduction of the problem $\texttt{SHAP}(\texttt{WA}, \texttt{MARKOV})$ to $\texttt{SHAP}_{1}(\texttt{WA}, \texttt{MARKOV})$ and $\texttt{SHAP}_{2}(\texttt{WA}, \texttt{MARKOV})$ is straightforward in light of equation \eqref{ffdist:decomposition}. The following lemma formally states this fact: 
\begin{lemma} \label{reduction-shap}
   If $\texttt{SHAP}_{1}(\texttt{WA}, \texttt{MARKOV})$ and $\texttt{SHAP}_{2}(\texttt{WA}, \texttt{MARKOV})$ are in \emph{FP}, then $\texttt{SHAP}(\texttt{WA}, \texttt{MARKOV})$  is in \emph{FP}.
\end{lemma}
\begin{proof}
    The proof is straightforwardly obtained from equation \eqref{ffdist:decomposition}. Assume $\texttt{SHAP}_{1}(\texttt{WA}, \texttt{MARKOV})$ and $\texttt{SHAP}_{2}(\texttt{WA}, \texttt{MARKOV})$ are in FP. Then, there exists two algorithms, say $\mathcal{A}_{1}$ and $\mathcal{A}_{2}$, that solve the problems $\texttt{SHAP}_{1}$ and $\texttt{SHAP}_{2}$ respectively in $O(\texttt{poly}(\texttt{size}(A), |w|,|\Sigma))$ time. 

    Fix an input of instance $<M,w,i,P>$ of $\texttt{SHAP}(\texttt{WA}, \texttt{MARKOV})$. To compute $\texttt{SHAP}(M,w,i,P)$ using $\mathcal{A}_{1},~\mathcal{A}_{2}$ as oracles, run the following schema: 
    \begin{enumerate}
        \item Call $\mathcal{A}_{1}$ on the set of input instances $\{<M,w,k,p>\}_{k \in [|w|]}$ yielding $\{y_{k}\}_{k \in [|w|]}$
        \item Call $\mathcal{A}_{2}$ on the set of input instances $\{<M,w,i,k,P>\}_{k \in [|w|]}$ yielding $\{y'_{k}\}_{k \in [|w|]}$
        \item Output: 
        $\sum\limits_{k=1}^{|w|-1} \frac{1}{k} (y_{k} - y'_{k})$
    \end{enumerate}
    The correctness of this schema to solve $\texttt{SHAP}(\texttt{WA}, \texttt{MARKOV})$ is guaranteed by equation \eqref{ffdist:decomposition}. In addition, by assumptions on $\mathcal{A}_{1},~ \mathcal{A}_{2}$, this schema runs also in $O(\texttt{poly}(\texttt{size}(A), |\Sigma|, |w|)$ time. 
\end{proof}

\subsection{Step 2: $\texttt{SHAP}_{1}(\texttt{WA}, \texttt{MARKOV})$ and $\texttt{SHAP}_{2}(\texttt{WA}, \texttt{MARKOV})$ are in \emph{FP}.}
This segment is dedicated to provide the outline of the proof of the following lemma:

\begin{lemma} \label{shap1shap2FP}
    The problems $\texttt{SHAP}_{1}(\texttt{WA}, \texttt{MARKOV})$ and $\texttt{SHAP}_{2}(\texttt{WA}, \texttt{MARKOV})$ are in \emph{FP}.
\end{lemma}

The result of the main theorem \ref{ffdist:mainthm} is an immediate corollary of lemma  \ref{reduction-shap} and lemma \ref{shap1shap2FP} presented in the previous segment of this section. 

The proof of lemma \ref{shap1shap2FP} will follow two steps. In the first step, the formulas of $\text{SHAP}_{1}$ and $\text{SHAP}_{2}$ will be reformulated in terms of operations over languages/seq2seq languages defined in section \ref{background}. The parametrization of these languages depends on the input instance of the problem. In the second step, we will show that WAs and WTs can be constructed in polynomial time that compute these languages/seq2seq languages. Combining the results of the two steps and the efficiency of implementing these operators for the case of WAs/WTs (lemma \ref{operators}), the proof of lemma \ref{shap1shap2FP} can be easily obtained. 

 \subsubsection{Step 2.a: Computation $\text{SHAP}_{1},~\text{SHAP}_{2}$ in terms of language operators.}
The following lemma provides a reformulation of the functions $\text{SHAP}_{1}$ and $\text{SHAP}_{2}$ in the form of operations over languages whose properties depend on the input instance of their respective problems:
\begin{lemma}\label{ffdist:mainlemma}
    Let $A$ be a WA over $\Sigma^{*}$,  a sequence $w \in \Sigma^{*}$, two integers $(i,k) \in [|w|]\times [|w|-1]$, and $P$ be an arbitrary probability distribution over $\Sigma^{\infty}$. We have 
    \begin{equation}\label{shap1wa}
        \text{SHAP}_{1}(A,w,k,P) = | f_{w,k} \otimes \Pi(f_{A}, g_{w,P}^{(1)}) |_{|w|}
    \end{equation}
    and,
    \begin{equation}\label{shap2wa}
         \text{SHAP}_{2}(A,w,i,k,P) = | f_{w,k} \otimes \Pi(f_{A},g^{(2)}_{w,i,P})|_{|w|}
    \end{equation}
    where 
    \vspace{-0.3cm}
    \begin{itemize}
        \item $f_{w,k} = \mathcal{P}_{k}^{w}$,
        \item $g^{(1)}_{w,P}$ is a seq2seq language over $\Sigma^{*} \times \Sigma_{\#}^{*}$ that satisfies the following constraint:
        \begin{small}
        \begin{equation} \label{constraint1}
            \forall (w',p) \in \Sigma^{|w|} \times \Sigma_{\#}^{|w|} :~~g_{w,P}^{(1)}(w',p) = P(w' | w' \in L_{p})
        \end{equation}
        \end{small}
        \item $g^{(2)}_{w,i,p}$ is a seq2seq language over $\Sigma^{*} \times \Sigma_{\#}^{*}$ that satisfies the following constraint:
        \begin{small}
         \begin{equation} \label{constraint2}
            \forall (w',p) \in \Sigma^{|w|} 
 \times \Sigma_{\#}^{|w|} :~~g_{w,i,P}^{(2)}(w',p) = P(w' | w' \in L_{p'})
        \end{equation}
        \end{small}
        where $p' = \texttt{swap}(p,i)$
    \end{itemize}
\end{lemma}
The proof is given in appendix \ref{mainlemma}.

 Expressions \eqref{shap1wa} and \eqref{shap2wa} reduce the problem of computing $\text{SHAP}_{1}$ and $\text{SHAP}_{2}$ to that of performing  operations over a language $f_{w,k}$ and two seq2seq languages, $g_{w,P}^{(1)},~g_{w,i,P}^{(2)}$ whose properties are given by equations \eqref{constraint1} and \eqref{constraint2} , respectively. The missing link to complete the proof of lemma \ref{shap1shap2FP} is to prove that a WA that implements the language  $f_{w,k}$, and WTs that compute seq2seq languages $g_{w,P}^{(1)},~g_{w,i,P}^{(2)}$ whose properties are given in lemma \ref{ffdist:mainlemma} can be constructed in polynomial time.
\subsubsection{Step 2.b: Construction of WAs/WTs that compute $f_{w,k},~g_{w,P}^{(1)},~g_{w,i,P}^{(2)}$ }
The key insight of the article is the following:

$~~$If $P \in \texttt{MARKOV}$, two seq2seq languages $g_{w,P}^{(1)}$ and $g_{w,i,P}^{(2)}$ that satisfy the constraints \eqref{constraint1} and \eqref{constraint2}, respectively, admit a representation using the WA/WT formalism. In addition, the construction of  WAs and WTs that compute these languages/seq2seq languages can be performed in time polynomial in the size of the input instance.

The next lemma provides a formal statement of this fact while also covering the language $f_{w,k}$.

\begin{lemma} \label{complex} 
    \begin{enumerate}
        \item \textbf{The language $f_{w,k}$:} 
        There exists an algorithm $\mathcal{A}_{1}$ that takes as input, a sequence $w \in \Sigma^{*}$, an integer $k \in [|w| - 1]$, runs in $O(\texttt{poly}(|w|))$, and outputs a WA $A_{k,w}$ over $\Sigma_{\#}^{*}$ that computes  the language $f_{w,k} = \mathcal{P}_{k}^{w}$. 
        \item \textbf{The seq2seq language $g_{w,P}^{(1)}$:} There exists an algorithm $\mathcal{A}_{2}$ that takes a sequence $w \in \Sigma^{*}$, and $P \in \texttt{MARKOV}$, runs in $O(\texttt{poly}(|w|, |\Sigma|))$, and outputs a WT $T_{w,P}$ that computes a  seq2seq language that satisfies the constraint \eqref{constraint1}.
        \item \textbf{The seq2seq language $g_{w,i,P}^{(2)}$:} There exists an algorithm $\mathcal{A}_{3}$ that takes as input a sequence $w \in \Sigma^{*}$, an integer $i \in  [|w|]$, and $P \in \texttt{MARKOV}$, runs in $O(\texttt{poly}(|w|, |\Sigma|))$, and outputs a WT that implements a seq2seq language over  that satisfies the constraint \eqref{constraint2}
        \end{enumerate}
        
\end{lemma}
  In the sequel, we shall refer to algorithms that compute $f_{w,k},~g_{w,P}^{(1)}$ and $g_{w,i,P}^{(2)}$ by $\mathcal{A}_{1},~\mathcal{A}_{2}$ and $\mathcal{A}_{3}$, respectively. 

 The proof of lemma \ref{complex} is constructive, and can be found in  appendix \ref{proofconstruction}. 
 
 The construction of $\mathcal{A}_{1}$ is relatively easy. As for $\mathcal{A}_{2}$ and $\mathcal{A}_{3}$, the key observation stems from the Bayes' formula:
  \begin{equation} \label{bayesfor}
  P(w |w \in L_{p}) = \frac{P(w) \cdot I_{L_{p}}(w)}{P(L_{p})}
  \end{equation}

 In light of the equation \eqref{bayesfor}, the construction of $\mathcal{A}_{2}$ and $\mathcal{A}_{3}$ will follow the same spirit of the main algorithm for solving $\texttt{SHAP}(\texttt{WA}, \texttt{MARKOV})$. In other words, it will involve the construction of a WT over $\Sigma^{*} \times \Sigma_{\#}^{*}$ that implements the language $I_{L_{p}}(w)$, and two  WAs over $\Sigma^{*}$ and $\Sigma_{\#}^{*}$ that implement the languages $P(w)$ and $\frac{1}{P(L_{p})}$, respectively. Since WAs/WTs are not closed under the division operation, the major difficulty in the construction lies in the design of a WA that implements the language $\frac{1}{P(L_{p})}$ involving a division operation. 

We note that since $\mathcal{A}_{1},~\mathcal{A}_{2}$ and $\mathcal{A}_{3}$ run in time polynomial in their respective input instances implies that the size of their output machines is also polynomial in the size of their input instance\footnote{FP $\subset$ FPSPACE}. This fact will appear explicitly in the constructive proof of lemma \ref{complex}.

   In light of lemma \ref{ffdist:mainlemma} and \ref{complex}, we are ready to prove the main lemma of this subsection: 

   \begin{proof}(lemma \ref{shap1shap2FP}) 
    We shall prove that $\texttt{SHAP}_{1}$ is in FP. A similar argument can be applied to derive the same result for $\text{SHAP}_{2}$.
    
      Define the following algorithmic schema that takes as input an instance $<A,w,k,P>$ where $A$ is a WA, $w \in \Sigma^{*}$, $i \in [|w|]$ and $P \in \texttt{MARKOV}$:
      \begin{enumerate}
          \item $A_{w,k} \leftarrow \mathcal{A}_{1} (w,k)$
          \item $T_{w,P} \leftarrow \mathcal{A}_{2}(w,P) $
          \item \textbf{Output:} 
          $|f_{A_{w,k}} \otimes \Pi(f_{A} , f_{T_{w,P}})|_{|w|}$
      \end{enumerate}
      By lemma \ref{ffdist:mainlemma} (equation \eqref{shap1wa}), and the properties of $\mathcal{A}_{1},~\mathcal{A}_{2}$ (lemma \ref{complex}), this schema solves exactly the problem $\texttt{SHAP}_{1}(\texttt{WA}, \texttt{MARKOV})$. 

      In addition, this schema also runs in $O(\texttt{poly}(\texttt{size}(A),|w|,|\Sigma|))$. Indeed, by lemma \ref{complex}, steps 1 and 2 run in $O(\texttt{poly}(|w|)$ and $O(\texttt{poly}(|w|, |\Sigma|))$, respectively. Consequently, by $\text{\emph{FP}} \subset \text{\emph{FP}SPACE}$, the size of their outputs $A_{w,k}$, and $T_{w,k,P}$  is also polynomial in $|w|$ and $|\Sigma|$. 
      
      On the other hand, given that the operators $\otimes, |.|_{n},~~\Pi$ over languages represented by WAs/WTs can be computed in polynomial time with respective to the size of their input instances (lemma \ref{operators}), this proves that the third step of the schema also runs in $O(\texttt{poly}(\texttt{size}(A), |w|, |\Sigma|))$ time.
      
   \end{proof}

\section{$\texttt{SHAP}(\texttt{D-DNF}, \texttt{MARKOV})$ and $\texttt{SHAP}(\texttt{DT}, \texttt{MARKOV})$ are in FP} \label{results:reduction}
In this section, we switch our focus to boolean functions, in particular the class of disjoint-DNFs (\texttt{d-DNF}) \footnote{For the general case of arbitrary DNFs, it has been shown that computing the SHAP score for this class of models when features are assumed to be independent is intractable under widely believed complexity assumptions \cite{arenas23}.}. The choice of this family of models is mainly motivated by the fact that it encompasses the family of decision trees, a central class of \textit{glass-box} models capturing substantial attention within the explainable AI community. Recent works have been dedicated to exploring the computation of SHAP scores for Tree-based models across diverse configurations \cite{lundbergnature,yang2022fast,arenas23,yu2023linear}. Later in this section, we shall prove that computing the SHAP score for the family of decision trees under Markovian distributions is reducible in polynomial time to $\texttt{SHAP}(\texttt{WA},\texttt{MARKOV})$, offering a polynomial-time algorithmic construction to compute the original SHAP score for the family of decision trees under the Markovian assumption. 

The class of disjoint-DNFs (d-DNFs) is formally defined as follows:
\begin{definition}[Disjoint DNF]
A d-DNF is a logical expression $\Phi(X_{1}, X_{2}, \ldots,X_{n})$ where $\{X_1, X_2, \ldots, X_n \}$ is a set of input boolean variables, such that:
\begin{itemize}
    \item $\Phi$ is expressed as a disjunction (logical OR) of clauses, where each clause is expressed as one or more conjunctions (logical AND) of literals.
    \item Each clause in the expression is mutually exclusive from the others, ensuring that for any input combination \((X_1, X_2, \ldots, X_n)\), only one clause evaluates to true.
\end{itemize}
\end{definition} 

\textbf{Example.} Let $X = \{X_{1},X_{2},X_{3}, X_{4}\}$ be a set of binary variables. The formula 
\begin{equation}\label{disjointdnf}
 \Phi = (X_{1} \land X_{3} \land X_{4}) \lor (\bar{X}_{1} \land X_{2} \land X_{3}) \lor (X_{2} \land \bar{X}_{3})
 \end{equation} 
 is a d-DNF over the variables $\{X_{i}\}_{i \in [4]}$ comprising 3 clauses. Indeed, for any two distinct clauses $(C_{i},C_{j})$ for $(i,j) \in [3]^{2}$, the intersection of the set of satisfying  variable assignments for $C_{i}$ and $ C_{j}$ is empty. 
 
A Markovian distribution $P$ over a boolean random vector of dimension $N$ is given as: 

$$P(X_{1}, \ldots , X_{N}) = P_{init}(X_{1}) \prod\limits_{i=1}^{N - 1} P_{i}(X_{i+1}|X_{i})$$

To avoid confusion with the sequential case, the set of Markovian distributions over boolean vectors shall be denoted $\overrightarrow{\texttt{MARKOV}}$. For an integer $N>0$, $\overrightarrow{\texttt{MARKOV}}_{N}$ will refer to the set of Markovian distributions over boolean vectors of dimension $N$.
 
The formal definition of the computational problem associated to compute the SHAP score of the class of d-DNFs under Markovian distributions is given as follows:

$\bullet$ \textbf{Problem:} $\texttt{SHAP}(\texttt{d-DNF}, \overrightarrow{\texttt{MARKOV}})$ \\
$~~$ \textbf{Instance:} A d-DNF $\Phi$ over $N$ boolean variables, an instance $\overrightarrow{x} \in \{0,1\}^{N}$, an integer $i \in [N]$, $P \in \overrightarrow{\texttt{MARKOV}}_{N}$ \\
$~~$ \textbf{Output:} Compute $\overrightarrow{\text{SHAP}}(\Phi, x, i,P)$

The complexity size of the input instance of this problem is given by the number of variables of $\Phi$, denoted $|\Phi|$, and the number of clauses in the d-DNF denoted $|\Phi|_{\#}$.

The claim of this section is given in the following theorem:
\begin{theorem}\label{ddnftheorem}
    $\texttt{SHAP}(\texttt{d-DNF}, \overrightarrow{\texttt{MARKOV}})$ is in \emph{FP}. 
\end{theorem}

The proof of theorem \ref{ddnftheorem} will proceed by reduction to the problem $\texttt{SHAP}(\texttt{WA}, \texttt{MARKOV})$.

Before providing the details of the reduction strategy, we shall present an interesting corollary of theorem \ref{ddnftheorem}, stating that the SHAP score computational problem for the family of decision trees under Markovian distributions is in FP.
\begin{corollary}
   Denote by $\texttt{DT}$ the set of decision trees computing boolean functions. The problem $\texttt{SHAP}(\texttt{DT}, \overrightarrow{\texttt{MARKOV}})$ is in FP.
\end{corollary}
\begin{proof}
    This result follows immediately from theorem \ref{ddnftheorem}, and the fact that given an arbitrary decision tree in $\texttt{DT}$, an equivalent d-DNF can be constructed in polynomial time with respect to the size of the decision tree (Property 1, \cite{aizenstein92}).
\end{proof}

\subsection{Proof of theorem \ref{ddnftheorem}: Reduction strategy}

Unlike WAs, d-DNFs compute boolean functions instead of languages. For the sake of the reduction, a first step consists at performing a  \textit{sequentialization} operation of the input instance of the problem $\texttt{SHAP}(\texttt{d-DNF}, \overrightarrow{\texttt{MARKOV}})$. We give next details of the construction.

$\bullet$ \textbf{Sequentialization of Markovian distributions:} For an integer $N > 0$, the sequentialization of a Markovian distribution in $\overrightarrow{\texttt{MARKOV}}_{N}$ to one in $\texttt{MARKOV}$ must ensure that both distributions are equal in the support $\{0,1\}^{N}$. Indeed, since the SHAP score of boolean functions over $N$ variables considers only the support $[N]$, the choice of the transition probability matrices for integers larger than $i>N$ can be set arbitrary, provided the resulting Markovian distribution remains polynomial-time computable. A possible sequentialization strategy of a distribution of $P$ in $\overrightarrow{\texttt{MARKOV}}_{N}$ is given by a  $\Tilde{P} \in \texttt{MARKOV}$ (which depends on $P$) such that:
$$\Tilde{P}_{init} = P_{init},~ \Tilde{P}_{i}(X_{i+1}|X_{i})  = \begin{cases}
    P_{i}(X_{i+1}|X_{i}) & \text{if  } i \in [N] \\
    P_{unif}(X_{i+1}) & \text{elsewhere}
\end{cases}$$
where $P_{unif}(X_{i+1})$ is the uniform distribution over $\{0,1\}$.

$\bullet$ \textbf{Sequentialization of d-DNFs.} 
For any integer $N>0$, and any boolean vector $\overrightarrow{X}$ over $\{0,1\}^{N}$, $\texttt{SEQ}(\overrightarrow{X})$ refers to the sequence $X_{1} \ldots X_{N}$ formed by the binary alphabet.

For a given d-DNF $\Phi$ over $N$ variables. Its sequential version is represented by the unweighted language $L_{\Phi}$ over $\Sigma^{*}$ such that:
$$L_{\Phi} \myeq \{w \in \{0,1\}^{|\Phi|}: \overrightarrow{X} = \texttt{SEQ}^{-1}(w) ~~\text{satisfies}~~ \Phi \}$$
Basically, $L_{\Phi}$ comprises the set of all satisfied assignments by the formula $\Phi$ arranged in a sequence. 
The following lemma is key to prove theorem \ref{ddnftheorem}. It establishes the existence of an algorithm that constructs in polynomial time a WA that computes the language $I_{\Phi}$.
\begin{lemma} \label{ddnfstowa}
    There exists an algorithm that takes as input a d-DNF $\Phi$, runs in time polynomial in $|\Phi|$ and $|\Phi|_{\#}$, and outputs a WA that computes the language $I_{L_{\Phi}}$. 
\end{lemma}

The proof of lemma \ref{ddnfstowa} can be found in appendix \ref{dnftowa}.

Next, we provide the proof of theorem \ref{ddnftheorem}.
\begin{proof}{(Theorem \ref{ddnftheorem})}
For an input instance $<\Phi, \overrightarrow{x}, i, P>$ of the problem $\texttt{SHAP}(\texttt{d-DNF}, \texttt{MARKOV})$. One can observe that:
\begin{equation} \label{reductionddng}
\overrightarrow{\text{SHAP}}(\Phi, x,i,P) = \text{SHAP}(I_{L_{\Phi}}, \texttt{SEQ}(\overrightarrow{x}), i , \Tilde{P})
\end{equation}

Equation \eqref{ddnftheorem} suggests the following polynomial-time reduction strategy from $\texttt{SHAP}(\texttt{d-DNF}, \overrightarrow{\texttt{MARKOV}})$ to $\texttt{SHAP}(\texttt{d-DNF}, \texttt{MARKOV})$:
\begin{enumerate}
    \item Construct a WA that computes the language $I_{L_{\Phi}}$. By lemma \ref{ddnfstowa}, this can be performed in $O(\texttt{poly}(|\Phi|, |\Phi|_{\#}))$ time.
    \item Apply the $\texttt{SEQ}(.)$ operation on $\overrightarrow{x}$.
    \item Wrap the parameters of $P$ in a machine implementing $\Tilde{P}$. For an input integer $i >0$, it tests whether $i > N$. If the answer is yes, it returns the uniform distribution. Otherwise, it returns $P_{i}$. The construction of this machine runs in $O(|\Phi|)$ time. In addition, the resulting Markovian distribution is polynomial-time computable.
\end{enumerate}
\end{proof}
\section{Conclusion} \label{conclusion}
In this article, we established the tractability of the SHAP score computational problem under the Markovian assumption for the family of weighted automata and the family of disjoint-DNFs which encompasses, up to a polynomial-time reduction, the family of decision trees. The proof is constructive and is readily amenable to a translation into a practical algorithm that extends TreeSHAP to handle the Markovian case.

In conclusion, we note that, by revisiting algorithms designed to generate WTs that compute the seq2seq languages $g_{w,P}^{(1)},~g_{w,i,P}^{(2)}$ (lemma \ref{complex}),  the algorithmic construction described in this article can be easily extended to adapt to higher-order markovian distributions, e.g.$n$-gram models \cite{Fink2014}, provided the order of the distribution is of reasonably small size.  

In feature research, we aim at exploring the possibility to extend the tractability of SHAP explanations for other families of models under the Markovian assumption. An interesting family to be considered as a natural extension is the class of Deterministic Decomposable Circuits whose SHAP score computation under the feature independence assumption has been proven to be in FP \cite{arenas23}.

\section*{Acknowledgements}
We would like to express our gratitude to the anonymous reviewers for their invaluable feedback and constructive comments, which greatly contributed to enhancing the quality and clarity of this paper. In particular, we are thankful for their contributions in directing our attention to works addressing the limitations of the SHAP score.

\section*{Impact Statement}
Overall, our research contributes to the burgeoning field of explainable artificial intelligence (XAI) by focusing on the SHAP score, a widely employed method for interpreting ML models. We believe that the newly introduced theoretical tools used in our work to substantiate our findings, particularly those that establish connections between boolean circuits and finite state automata, hold the potential to inspire the development of novel algorithms within the realm of formal XAI, thus contributing to advance the field as a whole. Ultimately, our work contributes to the ongoing efforts to promote transparency, accountability, and trustworthiness in AI systems, paving the way for their responsible deployment across various domains.

\nocite{langley00}

\bibliography{example_paper}
\bibliographystyle{icml2024}

\newpage
\appendix
\onecolumn

\newpage
\appendix
\onecolumn

\section{Proof lemma \ref{operators}} \label{app:operators}
Lemma \ref{operators} establishes the existence of efficient procedures to compute the product, partition constant and projection operators over languages/seq2seq languages computed by means of WAs/WTs. In the same spirit of all results provided in this article, we shall provide a constructive proof of this lemma. In the sequel, we fix two finite alphabets $\Sigma,~\Delta$.

The proof will rely on the notion of \textit{Kronecker product} between matrices. A brief recall of this latter is given in the following.   

$\bullet$ \textbf{The Kronecker product:} The Kronecker product between $A \in \mathbb{R}^{n \times m}$ and $B \in \mathbb{R}^{k \times l}$, denoted $A \otimes B$, is a matrix in $\mathbb{R}^{(n \cdot k) \times (m \cdot l)}$  
constructed as follows 
 $$A \otimes B = \begin{bmatrix}
     a_{1,1} \cdot B & a_{1,2} \cdot B & \dots & a_{1,m} \cdot B] \\ 
     a_{2,1} \cdot B & a_{2,2} \cdot B & \dots &  a_{2,m} \cdot B] \\
     \vdots & \vdots & \vdots & \vdots \\ 
     a_{n,1} \cdot B & a_{n,2} \cdot B & \dots & a_{n,m} \cdot B
 \end{bmatrix}$$
 where, for $(i,j) \in [n] \times [m]$ $a_{i,j}$ corresponds to element in the $i$-th row and $j$-th column of $A$.
 
 A useful property of the Kronecker product in our context is the mixed-product property:
 \begin{proposition}{(Proposition 2.1, \cite{kiefer13})} \label{mixedproduct}
      $~~$ \textit{Given $A ,~B,~C,~D$ matrices with judicious dimensions, we have $(A \cdot B) \otimes (C \cdot D) = (A \otimes C) \cdot (B \otimes D)$}
\end{proposition}
 
Next, we prove the result of the three points mentioned in the lemma:

$\bullet$ \textbf{The product language of WAs:} An important property of WAs is their closure under the product operation. This fact is classical in the theory of rational languages and has been proven in Sch\"{u}tzenberger's seminal paper \cite{schutz61} where WAs have been first introduced.

For the sake of completeness, the following proposition provides the details of the construction of a WA that computes the product of two languages represented by their WAs:
\begin{proposition}
    Let $A=  <\alpha, \{A_{\sigma}\}_{\sigma \in \Sigma}, \beta>$ and $A' = <\alpha', \{A'_{\sigma}\}_{\sigma \in \Sigma}, \beta'>$ be two WAs over $\Sigma^{*}$. \\
    The WA $A \otimes A' = <\alpha \otimes \alpha', \{A_{\sigma} \otimes A'_{\sigma}\}_{\sigma \in \Sigma}, \beta \otimes \beta'>$ over $\Sigma^{*}$ computes the language
    $$f_{A \otimes A'}(w) = f_{A}(w) \cdot f_{A'}(w)$$
    for any $w \in \Sigma^{*}$.
\end{proposition}
\begin{proof}
    Let $A=  <\alpha, \{A_{\sigma}\}_{\sigma \in \Sigma}, \beta>$, and $A' = <\alpha', \{A'_{\sigma}\}_{\sigma \in \Sigma}, \beta'>$ be two WAs. Denote by $A \otimes A'$ the WA $<\alpha \otimes \alpha', \{A_{\sigma} \otimes A'_{\sigma}\}_{\sigma \in \Sigma}, \beta>$.
    
    For an arbitrary string $w \in \Sigma^{*}$, we have:
    \begin{align*}
        f_{A}(w) \cdot f_{A'}(w) &= (\alpha^T \cdot \prod\limits_{i=0}^{|w|} A_{w_{i}} \cdot \beta) \cdot (\alpha'^{T} \cdot \prod\limits_{i=0}^{|w|} A'_{w_{i}} \cdot \beta') \\
        &= (\alpha^{T} \otimes \alpha'^{T}) \cdot \prod\limits_{i=1}^{|w|} (A_{w_{i}} \otimes A_{w'_{i}}) \cdot (\beta \otimes \beta') \\
        &= (\alpha \otimes \alpha')^{T} \cdot \prod\limits_{i=1}^{|w|} (A_{w_{i}} \otimes A_{w'_{i}}) \cdot (\beta \otimes \beta') \\
        &= f_{A \otimes A'}(w)
    \end{align*}
    where the second equality results from the mixed-product property of the Kronecker product (proposition \ref{mixedproduct}).
\end{proof}
The construction of the product WA runs in $O(|\Sigma| \cdot \texttt{size}^{2}(A) \cdot \texttt{size}^{2}(A'))$.

$\bullet$ \textbf{The partition constant operator of WAs:} 
The following proposition provides an implicit polynomial-time procedure that computes the quantity $|f_{A}|_{n}$ for a WA $A$.
\begin{proposition}
    Let $A = <\alpha, \{A_{\sigma}\}_{\sigma \in \Sigma}, \beta>$ be a WA over $\Sigma^{*}$ and an integer $n > 0$. We have: 
    $$|f_{A}|_{n} = \alpha^{T} \cdot (\sum\limits_{\sigma \in \Sigma} A_{\sigma})^{n} \cdot \beta$$
\end{proposition}
\begin{proof}
    Let $A = <\alpha, \{A_{\sigma}\}_{\sigma \in \Sigma}, \beta>$ be a WA over $\Sigma^{*}$ and an integer $n > 0$.
    We first prove by induction that for any $n>0$, we have 
    \begin{equation} \label{recurrence}
        \sum\limits_{w \in \Sigma^{n}} A_{w} = (\sum\limits_{\sigma \in \Sigma} A_{\sigma})^{n} 
    \end{equation} 

    The case $n=1$ is trivial.

    Assume the expression \eqref{recurrence} is true for an integer $n >0$. Let's prove it is also the case for $n+1$. We have:
    $$\sum\limits_{w \in \Sigma^{n+1}} A_{w} = \sum\limits_{w \in \Sigma^{n}} \sum\limits_{\sigma \in \Sigma} A_{w \sigma} = \sum\limits_{w \in \Sigma^{n}} A_{w} \cdot \sum\limits_{\sigma \in \Sigma} A_{\sigma} = (\sum\limits_{\sigma \in \Sigma} A_{\sigma})^{n+1}$$ 

    which proves the equality \eqref{recurrence}.
    
    Let $A = <\alpha, \{A_{\sigma}\}_{\sigma \in \Sigma}, \beta>$. For an integer $n>0$, we have:
    $$|f_{A}|_{n} = \alpha^{T} \cdot \sum\limits_{w \in \Sigma^{n}} A_{w} \cdot \beta = \alpha^{T} \cdot (\sum\limits_{\sigma \in \Sigma} A_{\sigma})^{n} \cdot \beta$$
    where the second result is obtained from \eqref{recurrence}.
\end{proof}
The complexity of implementing this operation is given as: $O(\texttt{size}(A)^{2}(|\Sigma| + \texttt{size}(A))$.

$\bullet$ \textbf{The projection operator:}
The following proposition provides a proof of the third point of lemma \ref{operators}:
\begin{proposition} \label{projection}
     Let $\Sigma,~\Delta$ be two finite alphabets. Let $A = <\alpha, \{A_{\sigma}\}_{\sigma \in \Sigma}, \beta>$ be a WA over $\Sigma^{*}$ ,and $T = <\alpha' , \{\bar{A}_{\sigma}^{\sigma'}\}_{\sigma \in \Sigma,~\sigma' \in \Sigma'}, \beta'>$ a WT over $\Sigma^{*} \times \Delta^{*}$. The WA $\Pi(A,T) = <\alpha \otimes \alpha', \{\sum\limits_{\sigma \in \Sigma} A_{\sigma} \otimes \bar{A}_{\sigma}^{\sigma'} \}_{\sigma' \in \Delta}, \beta \otimes \beta'>$ over $\Sigma^{*}$ computes the language $\Pi(f_{A} \otimes f_{T})$.
\end{proposition}

    \begin{proof}
    Let  $A$  be a WA  over $\Sigma^{*}$, and $T$ be a WT over $\Sigma^{*} \times \Delta^{*}$. 

    Define the  WA $\Pi(A,T) = <\alpha \otimes \alpha', \{\sum\limits_{\sigma \in \Sigma} A_{\sigma} \otimes \bar{A}_{\sigma}^{\sigma'} \}_{\sigma' \in \Delta}, \beta \otimes \beta'>$ constructed from $A$ and $T$.
    
    For an arbitrary $u \in \Delta^{*}$, we have 
    \begin{align*}
        \sum\limits_{w \in \Sigma^{|u|}} f_{A}(w) \cdot f_{T}(w,u) &=  \sum\limits_{w \in \Sigma^{|u|}} (\alpha^{T} \cdot \prod\limits_{i=1}^{|w|} A_{w_{i}} \cdot \beta ) \\
        &~~~~~~\cdot (\alpha'^{T} \cdot \prod\limits_{i=1}^{|u|} \bar{A}_{w_{i}}^{u_{i}} \cdot \beta') \\
        &= \sum\limits_{w \in \Sigma^{|u|}} (\alpha \otimes \alpha')^{T} \cdot (\prod\limits_{i=1}^{|u|} A_{w_{i}} \otimes \bar{A}_{w_{i}}^{u_{i}} ) \cdot (\beta \otimes \beta') \\
        &=  (\alpha \otimes \alpha')^{T} \cdot (\prod\limits_{i=1}^{|u|} \sum\limits_{\sigma \in \Sigma} A_{\sigma} \otimes \bar{A}_{\sigma}^{u_{i}} ) \cdot (\beta \otimes \beta') \\
        &= f_{\Pi(A,T)}(u)
    \end{align*}
\end{proof}

The complexity of the construction implicitly outlined in proposition \ref{projection} is $O(\texttt{size}(A)^{2} \times \texttt{size}(T)^{2} \times |\Sigma|)$.
\section{Proof lemma \ref{ffdist:mainlemma}} \label{mainlemma}
We'll show the expression \eqref{shap1wa}. The expression \eqref{shap2wa} can be obtained by mimicking the proof herein. \\
    Let $A$ be a WA, a sequence $w \in \Sigma^{*}$, an integer $k \in [|w| - 1]$, and $P$ be an arbitrary distribution over $\Sigma^{\infty}$. Let $f_{w,k}$ (resp. $g_{w,P}^{(1)}$) a language (resp. seq2seq language) whose properties are given in the statement of the lemma. We have: 
    \begin{align*}
        \text{SHAP}_{1}(A,w,i,k,P) &=  \mathbb{E}_{p \sim \mathcal{P}_{k}^{w}} \mathbb{E}_{w' \sim P^{(|w|)}} [f_{A}(w') | w' \in L_{p}] \\ 
    &= \sum\limits_{p \in \Sigma_{\#}^{|w|}} \mathcal{P}_{k}^{w}(p) \sum\limits_{w' \in \Sigma^{|w|}} f_{A}(w') \cdot P(w' | w' \in L_{p}) \\
    &= \sum\limits_{p \in \Sigma_{\#}^{|w|}} f_{w,k}(p) \cdot \sum\limits_{w' \in \Sigma^{|w|}} f_{A}(w') \cdot g_{w,P}^{(1)}(w',p) \\
    &= \sum\limits_{p \in \Sigma_{\#}^{|w|}} f_{w,k}(p) \cdot \Pi(f_{A} , g_{w,P}^{(1)})(p) \\
    &= \sum\limits_{p \in \Sigma_{\#}^{|w|}} f_{w,k}(p) \cdot \Pi(f_{A} , g_{w,P}^{(1)})(p) \\
    &= | f_{w,k} \otimes \Pi(f_{A}, g_{w,P}^{(1)})|_{|w|} 
    \end{align*}

\section{Proof lemma \ref{complex}} \label{proofconstruction}
The core statement of lemma \ref{complex} encompasses three results stating the existence of three efficient algorithmic procedures, namely $\mathcal{A}_{1},~\mathcal{A}_{2}$ and $\mathcal{A}_{3}$, that construct a collection of WAs/WTs whose characteristics are given in the lemma statement. 

This appendix will be split into two segments. The first segment furnishes the algorithmic construction of $\mathcal{A}_{1}$. Due to the close similarities of algorithms $\mathcal{A}_{2},~\mathcal{A}_{3}$, they shall be treated simultaneously in the second segment. 

Before outlining these constructions, we furnish a brief recall of some sub-families of WAs and WTs serving as a technical background on top of which the proof will be built. In particular, three sub-families will be introduced: Determinstic Finite Automata, deterministic WAs, and Deterministic Finite Transducers. 

In the sequel, we fix an alphabet $\Sigma,~\Delta$.

$\bullet$ \textbf{Deterministic finite Automata.} The class of deterministic finite automata (DFAs) is a popular sub-family of WAs adapted to model unweighted languages. A DFA is formally represented by a tuple $<Q, q_{init}, \delta, F>$, where:

\begin{itemize}
    \item $Q$ is a finite set of states,
    \item $q_{init} \in Q$ is called the initial state,
    \item $\delta: Q \times \Sigma \rightarrow Q$ is a partial function \footnote{A partial function $f$ from a set $X$ to $Y$ is a function whose input domain is a subset of $X$ (i.e. it doesn't necessarily assign an output to every element of $x$)} called the transition function,
    \item $F \subseteq Q$ is called the the set of final states,
\end{itemize}

For a DFA $A = <Q, q_{init}, \delta, F>$, a valid path  over $A$ labeled by a sequence $w \in \Sigma^{*}$ is a sequence of state-symbol pairs taking the form:
$q_{0} w_{1} q_{1} \ldots w_{|w|} q_{|w|}$, such that  for any $i \in \{0, \ldots |w|- 1\}: \delta(q_{i}, w_{i+1}) = q_{i+1}$. A valid path labeled by $w$ is said to be accepting if $q_{0} = q_{init}$ and $q_{|w|} \in F$.  

An important property of DFAs lies in that the cardinality of the set of its valid paths labeled by an arbitrary sequence $w \in \Sigma^{*}$ is at most equal to $1$. The unweighted language accepted by a DFA corresponds to the set of sequences that label a valid accepting path over the DFA. 

$\bullet$ \textbf{Deterministic Finite Transducers.} Deterministic Finite Transducers (DFTs) represent the analogous counterpart of DFAs adapted to seq2seq languages, and constitutes a sub-family of WTs that compute unweighted seq2seq languages. A DFT over $\Sigma \times \Delta$ is formally represented by a tuple $<Q, q_{init}, \delta, F>$ 
\begin{itemize}
    \item $Q$ is a finite set of states.
    \item $q_{init}$ is the initial state.
    \item $\delta: Q \times \Sigma \rightarrow \Delta \times Q$ is a partial function called the transition function.
    \item $F$ is called the set of final states.
\end{itemize}

The formal description of a DFT resembles to that of DFAs, and operates in a closely similar manner. 

For a DFT $T = <Q,q_{init}, \delta, F>$, a valid path over $T$ labeled by a pair of sequences $(u,v) \in \Sigma^{*} \times \Delta^{*}$ such that $|u| = |v| = n$ is a sequence of elements in $Q \times \Sigma \times \Delta$ taking the form: $q_{0}u_{1} v_{1} q_{1} \ldots u_{n} v_{n} q_{n}$ where for any $i \in \{0, \ldots n- 1\}: \delta(q_{i}, u_{i+1}) = (v_{i+1},q_{i+1})$. A valid path over $(u,v) \in \Sigma^{*} \times \Delta^{*}$ is said to be accepting if $q_{0} = q_{init}$ and $q_{n} \in F$.

DFTs enjoy a similar property than DFAs in that for any pair of sequences over $\Sigma^{*} \times \Delta^{*}$ with the same length, there exists at most one valid path labeled by this pair. The unweighted seq2seq language accepted by a DFT is equal to the set of sequence pairs over $\Sigma^{*} \times \Delta^{*}$ that label a valid accepting path. 

$\bullet$ \textbf{Deterministic Weighted Automata.} DWAs is the weighted variant of DFAs. It aligns with the structure of DFA while augmenting its transitions with real-valued weights. Formally, a DWA is defined as follows:
\begin{itemize}
    \item $Q$ is a finite set of states,
    \item $q_{init} \in Q$ is called the initial state,
    \item $W: Q \times \Sigma \rightarrow Q \times \mathbb{R}$ is a partial function \footnote{A partial function $f$ from a set $X$ to $Y$ is a function whose input domain is a subset of $X$ (i.e. it doesn't necessarily assign an output to all elements of $x$)} called the weight function,
    \item $F \subseteq Q$ is called the the set of final states,
\end{itemize}

Similar to DFAs, any sequence $w \in \Sigma^{*}$ labels at most a valid path, where the notion of a valid path is equivalent to that of DFAs. However, unlike DFAs, valid paths are assigned real-valued weighted instead of the boolean notion of acceptability. The weight assigned to a path starting from the initial state $q_{init}w_{1} \ldots q_{n-1}w_{n-1}q_{n}$ is equal to:

$$\cdot \prod\limits_{i=1}^{n-1}W(q_{i},w_{i})[2] \cdot I_{F}(q_{n})$$
where $W(q,\sigma)[2]$ refers to the weight associated to the transition $\delta(q,\sigma)$.

This weight coincides with the value assigned to the sequence $w_{1} \ldots w_{n}$ by the seq2seq language computed by the WT. Sequences that label no valid path are assigned the weight $0$ by default.

After presenting this brief technical background, we are now ready to prove the core statement of the lemma:

\subsection{Construction of $\mathcal{A}_{1}$.}
Recall that $\mathcal{A}_{1}$ refers to an algorithm that takes as input a string $w \in \Sigma^{*}$, an integer $k \in [|w|]^{2}$, runs in $O(\texttt{poly}(|w|))$, and outputs a WA over $\Sigma_{\#}^{*}$ that computes the language $\mathcal{P}_{k}^{w}$. The probability distribution $\mathcal{P}_{k}^{w}$ refers to the uniform distribution over the set of patterns:
 $$\mathcal{L}_{k}^{w} \myeq \{ p \in \Sigma_{\#}^{|w|}: ~~|p|_{\#} = k \land w \in L_{p} \}$$

  The algorithmic construction of $\mathcal{A}_{1}$ aligns with two sequential steps:
  \begin{enumerate}
      \item Create a DFA over $\Sigma_{\#}^{*}$ that accepts the language $\mathcal{L}_{k}^{w}$,
      \item Normalize the resulting DFA by the quantity $\frac{1}{|\mathcal{L}_{k}^{w}|}$ to obtain the output WA. Note that $|\mathcal{L}_{k}^{w}|$ is equal to $\frac{|w|!}{ (k)! \cdot (|w| - k)!}$ and can be computed in $O(\texttt{poly}(|w|))$ time. 
  \end{enumerate} 
  
  The second step of the algorithmic construction, i.e. the normalization step, is straightforward. Indeed, given a WA $A = <\alpha, \{A_{\sigma}\}_{\sigma \in \Sigma}, \beta>$ and a normalizing constant $C \in \mathbb{R}$, the WA $A' = <C \cdot \alpha, \{A_{\sigma}\}_{\sigma \in \Sigma}, \beta>$ computes the (normalized) language $f_{A'} = C \cdot f_{A}$. In addition, it's easy to observe that this operation can be performed in polynomial time with respect to the size of $A$. Our claim, that we shall prove next in this subsection, is that the size of the DFA $A$ is $O(\texttt{poly}(|w|))$. Assuming this claim holds, the normalization operation runs in $(\texttt{poly}(|w|)$.

  The rest of this subsection will focus on the first step of the algorithmic construction: 
  
$\bullet$ \textit{Creation of a DFA that accepts the language $\mathcal{L}_{k}^{w}$:}\\ 
Fix an input instance $w \in \Sigma^{*}$, $k \in [|w|]$. A key observation for the DFA construction consists at noting that, during a forward processing run over an input pattern to check its membership in $\mathcal{L}_{k}^{w}$, a sufficient information to keep of the run's history is summarized in the following:
  \begin{itemize}
      \item \textit{The position of the next symbol:} This information is useful to ensure that the input pattern satisfies the constraint $w \in L_{p}$ imposed by definition of $\mathcal{L}_{k}^{w}$. Additionally, this information will enable rejecting the patterns whose length is greater than $|w|$. In our case, this information lies in the interval $\{0,1, \ldots, |w|\}$,
      \item \textit{The number of occurrences of the symbol $\#$ in the processed prefix of the input pattern:} This information enables to ensure that only patterns that satisfies the constraint $|p|_{\#} = k$ will be accepted. In our case, this information lies in the range $\{0,1, \ldots, k \}$. 
  \end{itemize}
 In light of this discussion, the construction of the  DFA that accepts the language $\mathcal{L}_{k}^{w}$:
 \begin{itemize}
\item \textbf{The state space:} $Q = \{0,1, \ldots ,|w|] \times \{0, 1, \ldots , k\}$ 
\item \textbf{The initial state:} $q_{init}=(0,0)$. The first element of the pair signifies that the forward run is at position $0$ (i.e. no symbol in the input pattern has been processed so far). The second element signifies that $0$ occurrences of the symbol $\#$ has been encountered in the processed input pattern so far.
\item \textbf{The transition function:} 
 For a state $(l,l') \in \{0, \ldots, |w| -1\} \times \{0, \ldots, k - 1 \}$
 
  \textit{Case 1 ($p_{l+1} = \#$).} We increment both the number of occurences of $\#$ in the input pattern and the position of the sequence by 1 which entails a transition to $(l+1,l'+1)$:
  $$\delta((l,l') , \#) =  (l+1, l'+ 1)$$ 

  \textit{Case 2 ($p_{l+1} = w_{l+1}$).} we increment the position of the input pattern to $l+1$ without incrementing the number of occurrences of $\#$. 
  $$\delta((l,l'), w_{l+1})) = (l+1,l')$$

  No other transitions are added to the transition map for all the other cases. 
  
 \item \textbf{The final set of states:}  $F = \{(|w| , k)\}$.
\end{itemize}
One can check that the complexity of this algorithmic construction runs in $O(|w|^{2})$ time.

\subsection{Constructions of $\mathcal{A}_{2}$ and $\mathcal{A}_{3}$.}
Due to the close similarities in the construction of algorithms $\mathcal{A}_{2}$ and $\mathcal{A}_{3}$, we dedicate this segment to treat both algorithms simultaneously. The presence of the $\texttt{swap}(.)$ operation in $\mathcal{A}_{3}$ brings an additional difficulty to this latter, when compared to $\mathcal{A}_{2}$. Consequently, we choose to treat $\mathcal{A}_{3}$ as a main case. The subtle differences between $\mathcal{A}_{2}$ and $\mathcal{A}_{3}$ will take the form of notes where these differences will be highlighted. 

In lemma \ref{ffdist:mainlemma}, $\mathcal{A}_{3}$ designates an algorithm that takes as input a string $w \in \Sigma^{*}$, an integer $i\in [|w|]$, a probability distribution $P \in \texttt{MARKOV}$, and outputs a WT $T_{w,P}$ over $\Sigma^{*} \times \Sigma_{\#}^{*}$ that computes a seq2seq language that satisfies the following constraint:
\begin{equation}
\forall(w',p) \in \Sigma^{|w|} \times \Sigma_{\#}^{|w|}: f(w',p) =  P(w' | w' \in L_{\texttt{swap}(p,i)})
\end{equation}

 Instead of this formulation, we'll exploit an equivalent re-expression of the constraint in the algorithmic design obtained using Bayes' rule:
\begin{equation} \label{algo2constraint}
    \forall(p,w') \in \Sigma^{|w|} \times \Sigma_{\#}^{|w|}: f(w',p) = \frac{P(w') \cdot I_{L_{\texttt{swap}(p,i)}}(w')}{P(L_{\texttt{swap}(p,i)})}
\end{equation} 

The algorithms $\mathcal{A}_{2}$, and $\mathcal{A}_{3}$ will be designed following the same paradigm employed to construct the main algorithm for solving $\texttt{SHAP}(\texttt{WA}, \texttt{MARKOV})$. Specifically, it will involve the construction of WAs/WTs that compute languages dependent on the input instance of the problem. Then, the application of efficiently computable operators over these constructed WAs/WTs will yield a WT that satisfies the constraint \eqref{algo2constraint}.

Besides operators introduced in section \ref{background}, namely the product operator, the partition constant operator and the projection operator, we shall introduce two additional operators over seq2seq languages which will be useful in this context. An emphasis will be put on the computational efficiency of implementing these operators for the case of seq2seq languages represented by WTs. 

Fix two finite alphabets $\Sigma$ and $\Delta$.

$\bullet$ \textbf{The inverse operator:} The inverse operator takes as input a language a seq2seq language $\Sigma^{*} \times \Delta^{*}$ $f$, and returns the seq2seq language denoted $\texttt{inv}(f)$ such that:
$$\texttt{inv}(f)(u,s) \myeq f(s,u)$$
for $(u,s) \in \Sigma^{*} \times \Delta^{*}$ such that $|u| = |s|$.

This operator settles for performing a swap operation of the arguments given to compute the seq2seq language for a given pair of sequences.

When a seq2seq language over $\Sigma^{*} \times \Delta^{*}$ is computed by a WT $T = <\alpha, \{A_{\sigma}^{\sigma'}\}_{(\sigma, \sigma') \in \Sigma \times \Delta}, \beta>$, the WT that computes the seq2seq language $\texttt{inv}(f_{T})$ can be trivially obtained as $<\alpha, \{A_{\sigma'}^{\sigma}\}_{(\sigma',\sigma) \in \Delta \times \Sigma}, \beta>$.

$\bullet$ \textbf{The multiplicative operator:} This operator, which we'll refer to as \textit{the multiplicative operator}, takes as input a language $f$ over $\Sigma^{*}$ and a seq2seq language $g$ over $\Sigma^{*} \times \Delta^{*}$, and outputs a seq2seq language over $\Sigma^{*} \times \Delta^{*}$, denoted $f \times g$ such that
\begin{equation} \label{atimesb}
(f \times g) (u,s) = f(u) \cdot g(u,s)
\end{equation}
for any $(u,s) \in \Sigma^{*} \times \Delta^{*}$ such that $|u| = |s|$.

When the language $f$ and the seq2seq language $g$ given as arguments to this operator are represented by a WA $A$ and a WT $T$, respectively, then $f \times g$ can be computed by a WT. Moreover, the construction of this WT can be performed in time polynomial in the size of $A$ and $T$. The followin proposition provides a proof of this fact:
\begin{lemma}
    Let $A = <\alpha, \{A_{\sigma}\}_{\sigma \in \Sigma}, \beta>$ be a WA over $\Sigma^{*}$, $T = <\alpha' , \{B_{\sigma}\}_{\sigma \in \Sigma}^{\sigma' \in \Delta}, \beta'>$ a WT over $\Sigma^{*} \times \Delta^{*}$.
    
    The WT $A \times T = <\alpha \otimes \alpha', \{A_{\sigma } \otimes B_{\sigma}^{\sigma'}\}_{\sigma \in  \Sigma}^{\sigma' \in \Delta}, \beta \otimes \beta'>$ over $\Sigma^{*} \times \Delta^{*}$ computes the seq2seq language $f_{A \times T}$.
\end{lemma}
\begin{proof}
    Let $A = <\alpha, \{A_{\sigma}\}_{\sigma \in \Sigma}, \beta>$ be a WA over $\Sigma^{*}$, $B = <\alpha' , \{B_{\sigma}\}_{\sigma \in \Sigma}^{\sigma' \in \Delta}, \beta'>$ a WT over $\Sigma_{*} \times \Delta^{*}$. Let $A \times B = <\alpha \otimes \alpha', \{A_{\sigma } \otimes B_{\sigma}^{\sigma'}\}_{\sigma \in  \Sigma}^{\sigma' \in \Delta}, \beta \otimes \beta'>$ be the constructed WT from $A$ and $B$. 
    
    Fix a pair  $(u,s) \in \Sigma^{*} \times \Delta^{*}$ such that $|u| = |s|$. We have  
    \begin{align*}
        f_{A}(u) \cdot f_{T}(u,s) &= (\alpha^{T} \cdot \prod\limits_{i=1}^{|u|} A_{u_{i}} \cdot \beta) \cdot (\alpha'^{T} \cdot \prod\limits_{i=1}^{|w|} B_{w_{i}}^{s_{i}} \cdot \beta') \\
        &= (\alpha \otimes \alpha')^{T} \cdot \prod\limits_{i=1}^{|u|} (A_{u_{i}} \otimes B_{w_{i}}^{s_{i}})  \cdot (\beta \otimes \beta') \\
        &= f_{A \times B}(u,s)
    \end{align*}
    where the second equality is an application of the mixed product property of the Kronecker product (proposition \ref{mixedproduct}).
\end{proof}

After introducing the inverse and the multiplicative operator, we are now ready to provide the overall structure of algorithms $\mathcal{A}_{2}$ and $\mathcal{A}_{3}$. 

Fix an input instance $w \in \Sigma^{*}$, a pair of integers $i \in [|w|]$, and $P \in \texttt{MARKOV}$.

The algorithm $\mathcal{A}_{3}$ will follow three  steps:
\begin{itemize}
    \item \textit{Step 1:} Construct a DWA, denoted $A_{w,P}$, over $\Sigma^{*}$ that computes the language
    \begin{equation} \label{Awp}
    f_{A_{w,P}}(w') = \begin{cases}
    P(w') & \text{if } w' \in \Sigma^{|w|} \\
    0 & \text{elsewhere}
    \end{cases}
    \end{equation}
    \item \textit{Step 2:} Construct a DFT, denoted $T_{i}$, over $\Sigma^{*} \times \Sigma_{\#}^{*}$ that computes the (unweighted) seq2seq language:
    \begin{equation} \label{twi}
    f_{T_{w,i}}(w',p) = I_{L_{\texttt{swap}(p,i)}}(w')
    \end{equation}
    for any pair $(w',p) \in \Sigma^{*} \times \Sigma_{\#}^{*}$.
    \item \textit{Step 3:} Construct a DWT over $\Sigma_{\#}^{*}$, denoted $A_{w,i,P}$ that computes a language over $\Sigma_{\#}^{*}$ such that: 
    \begin{equation} \label{twip}
    f_{A_{w,i,P}}(p) = \frac{1}{P(L_{\texttt{swap}(p,i)})}
    \end{equation}
    for any $p \in \Sigma_{\#}^{|w|}$.
\end{itemize}

Assume we have the WAs $A_{w,P}, A_{w,i,P}$ and the WT $T_{w,i}$ that compute languages/seq2seq languages described in steps 1, 2 and 3, respectively. In light of the equation \eqref{algo2constraint}, 
the seq2seq language computed by the WT 
$$\texttt{inv}(A_{w,i,P} \times \texttt{inv}(A_{w,P} \times T_{i}))$$ 
 satisfies the constraint of the seq2seq language $g_{w,i,P}^{(1)}$. This resulting WT represents the output of $\mathcal{A}_{3}$.

$\bullet$ \textbf{Note.} \textit{At this stage, a slight difference between $\mathcal{A}_{2}$ and $\mathcal{A}_{3}$ lies in steps 2 and 3. For the case of $\mathcal{A}_{2}$, the pattern $\texttt{swap}(p,i)$ should be replaced by $p$ in equations \eqref{twi} and \eqref{twip}, in which case a different DFT and DWT have to be designed to compute these set of languages/seq2seq languages. Later in this segment, we shall highlight their construction.}

It's left to show how to construct these three machines in polynomial time with respect to the size of the input instance. The constructions of $A_{w,P}$ and $T_{w,i}$ are relatively easy. The construction $A_{w,i,P}$ is more challenging.

The remainder of this section will be split in three segments, each of which is dedicated to provide the implementation details of one of the steps of the algorithmic structure outlined above.

\subsubsection{Step 1: Construction of $A_{w,P}$.}
$A_{w,P}$ refers to the WA that computes the language expressed in \ref{Awp}. 

Given a string $w \in \Sigma^{*}$ and $P \in \texttt{MARKOV}$. A Markovian distribution over the finite support $\Sigma^{|w|}$ can be easily simulated by a DWA. The construction consists at maintaining in the state memory of the DWA the position reached so far in the sequence and the last generated symbol. These two pieces of information are sufficient to simulate a Markovian distribution.

For the sake of the construction, we add a new symbol, denoted $<BOS>$, that refers to the beginning of a sequence. 

The outline of the construction is given as follows:
\begin{itemize}
    \item \textit{The state space:} $Q = \{0,1,..,|w|\} \times (\Sigma \cup < \text{BOS}>)$,
    \item \textit{The initial state:}   $q_{init} = (0, <\text{BOS}>)$
   \item  \textit{The weight function:} Let $q=(i, \sigma)$ be a state in $Q$. We denote by $\sigma'$ an arbitrary symbol in $\Sigma$. We distinguish between two cases: 
    \begin{itemize}
        \item \textit{Case 1 $\left((i,\sigma) = (0, <BOS>)\right)$:} 
        $$W((0,<BOS>), \sigma') = \left((1, \sigma'), P_{init} (\sigma')\right)$$
        \item \textit{Case 2 $\left(i < |w|\right)$:}
         $$W((i,\sigma), \sigma') =( (i+1, \sigma') = P_{i}(\sigma'|\sigma))$$
    \end{itemize}
    \item \textit{The final weight vector:} $F = \{(|w|, \sigma):~ \sigma \in \Sigma \}$
\end{itemize}
A valid path labeled by a sequence $w' \in \Sigma^{|w|}$ over the constructed DWA is given as: $$(0,<BOS>)w'_{1}(1,w'_{1}) \ldots  (|w| - 1, w'_{|w|-1}) w'_{|w|} (|w|, w'_{|w|})$$
The weight of this path is equal to $P_{init}(w'_{1}) \cdot \prod\limits_{i=1}^{|w|-1} P_{i} (w'_{i+1}|w'_{i}) \cdot I_{F}\left((|w|, w'_{|w|})\right) = P(w')$.

Provided $P$ is polynomial-time computable, this construction runs in $O(\texttt{poly}(|w|, |\Sigma))$ time.
\subsubsection{Step 2: Construction of $T_{i}$.}
Given an integer $i > 0$, the goal is to construct a DFT $T_{i}$ over $\Sigma^{*} \times \Sigma_{\#}^{*}$ that computes the seq2seq language whose expression is given in \eqref{twi}. 

The construction is relatively easy. The state of the DFT will keep in its memory the current position of the pair of sequences being parsed up to position $i$. At a position $j<i$, the DFT will enable a transition from a state $j$ to a state $j+1$ if and only if the current pair of symbols to parse $(w'_{j+1}, p_{j+1})$ satisfies the constraint $(w'_{j+1} = p_{j+1}) \lor p_{j+1} = \#)$,. For the particular case, $j=i-1$, where the swap operation needs to be taken into account, a transition is allowed to $j+1$ regardless of the pair of symbols $(p_{i},w'_{i})$ fed to the DFT. 

\begin{figure}\label{fig:Ti}
     \centering
    \begin{tikzpicture}[node distance = 2.5cm, on grid, auto]
    \node (q0) [state, accepting, initial] {$0$};
    \node (q1) [state, accepting, right = of q0]  {$1$};
    \node (q2) [state, accepting, right = of q1]  {$2$};
    \node (q3) [state, accepting, right = of q2] {$3$};
    \path[->] (q0) edge [above] node {$\sigma : \sigma$} (q1);
    \path[->] (q0) edge [below] node {$\sigma : \#$} (q1);

    \path[->] (q1) edge [above] node {$\sigma : \sigma$} (q2);
    \path[->] (q1) edge [below] node {$\sigma : \#$} (q2);

    \path[->] (q2) edge [above] node {$\sigma : \sigma'$} (q3);

    \path[->] (q3) edge [loop above] node {$\sigma : \sigma,~\sigma: \#$} ();
\end{tikzpicture}
 \caption{A DFT $T_{i}$ that computes the seq2seq language $g(w',p) = I_{L_{\texttt{swap}(p,i)}}(w')$ for $i = 3$. $\sigma$ (resp. $\sigma'$) refers to any symbol in $\Sigma$ (resp. $\Sigma_{\#}$. \label{fig:Ti}}
\end{figure}
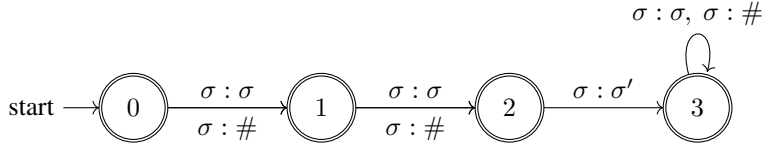

The formal description of a DFT $T_{i}$ is given in the following. An illustrative example of this construction is given in figure \ref{fig:Ti}.
\begin{itemize}
\item \textit{The state space:} $Q = \{0,1, \ldots, i\}$
\item \textit{The initial state:} $q_{init} = 0$,
\item \textit{The transition function:} Let $j$ be a state in $Q$. We distinguish between three cases:
\begin{enumerate}
    \item \textit{Case 1 ($j < i - 1 $):}
    $$\delta(j, (\sigma,\sigma')) = j+1$$
    for $(\sigma, \sigma') \in \Sigma \times \Sigma_{\#}$ such that ($\sigma = \sigma' \lor \sigma' = \#$)
    \item \textit{Case 2 ($j=i-1$):} 
    $$\delta(j, (\sigma, \sigma')) = j +1$$
    for any pair of symbols $(\sigma,\sigma') \in \Sigma \times \Sigma_{\#}$
    \item \textit{Case 3 ($j = i$):}
    $$w(i, (\sigma, \sigma')) = i$$
    for $(\sigma, \sigma') \in \Sigma \times \Sigma_{\#}$ such that ($\sigma = \sigma' \lor \sigma' = \#$)
\end{enumerate}
 \item \textit{The set of final states:} $Q$.  
\end{itemize}

$\bullet$ \textbf{Note.} \textit{For the case of the algorithm $\mathcal{A}_{2}$, a DFT that computes the seq2seq language $f(p,w') = I_{L_{p}}(w')$ is a  trivial single-state DFT that settles for testing at each step during the forward run whether the pair of input symbols $(\sigma, \sigma') \in \Sigma \times \Sigma_{\#}$ satisfies the constraint: $\sigma' = \sigma \lor \sigma' = \#$. 
}

\subsubsection{Step 3: Construction of $A_{w,i,P}$.}
In the remainder of this segment, we fix a string $w \in \Sigma^{*}$, and an integer $i \in [|w|]$, and $P \in \texttt{MARKOV}$.

Recall that the DWA $A_{w,i,P}$ over $\Sigma_{\#}^{*}$ is required to compute a language that satisfies the constraint \eqref{twip}. The construction of the DWA $A_{w,i,P}$ is more challenging than the construction pf $A_{w,P}$ and $T_{i}$ detailed in previous segments. The difficulty lies in the fact that, unlike the product operation, the set of WAs is not closed under the division operation.

By means of Bayes' rule, the constraint \eqref{twip} is explicitly given as
\begin{equation} \label{computesipw}
\forall p \in \Sigma_{\#}^{|w|}:~~f_{A_{w,i,P}}(p) = \frac{1}{P_{init}(w'_{1} \in L_{\texttt{swap}(p,i)_{1}})} \cdot \frac{1}{ \prod\limits_{j=1}^{|w|-1} P (w'_{j+1} \in L_{\texttt{swap}(p,i)_{j+1}}|w'_{1:j} \in L_{\texttt{swap}(p,i)_{1:j}})}
\end{equation}
 
When trying to construct a DWA that satisfies the formula \eqref{computesipw}, a difficulty arises by noting that the product terms forming the right-side of the equation requires maintaining the full history of the input pattern. A construction of a DWA that na\"ively simulates the equation \eqref{computesipw} would have a state space whose size is $O(|\Sigma|^{|w|})$. 

To circumvent this issue, an intermediary question to raise is concerned with the size of the minimal sufficient information to hold about a running pattern $p_{1,j}$ to compute the quantity $P(w'_{j+1} \in L_{p_{j+1}}|w'_{1:j} \in L_{p_{1:j}})$. Under the assumption that $P \in \texttt{MARKOV}$, one can observe that the minimal sufficient information to retain about the past of a pattern during a forward run is:
\begin{enumerate}
    \item The current position in the processed sequence.
    \item The last position where a symbol $ \sigma \in \Sigma$ has been encountered during the processing run. 
    \item The symbol that holds the position described in the previous point.
\end{enumerate}

To gain some intuition on the points discussed above, we provide an illustrative example:

$\bullet$ \textbf{Example:} Let $\Sigma = \{a,b\}$ be an alphabet, and $P \in \text{MARKOV}$. Let $p = a\#  a\#b$ be a pattern (the support is equal to $5$). 
Let's fix as a goal the computation of the quantity $P(w \in L_{a \# a \# b})$. Using Bayes' rule, we have 
$$P(w \in L_{a \# a \# b}) = P(w_{5} = b | w_{1} = a \land w_{3} = a) \cdot P(w_{3} = a | w_{1}=a) \cdot ¨P(w_{1} = a)$$
Since $P \in \text{MARKOV}$, $w_{5}$ is independent of $w_{1}$ given $w_{3}$. Thus, 
$$P(w \in L_{a \# a \# b}) = P(w_{5} = b | w_{3} = a) \cdot P(w_{3} = a | w_{1} = a) \cdot P(w_{1} = a)$$
Note that each product term in the right side of the equation depends only the current position, the last position where a symbol different than $\#$ has been encountered and the symbol found in this position. 

The points 2 and 3 are formalized by introducing the following two functions:  
\begin{itemize}
\item The $\texttt{pos}(.)$ function:
\begin{align}
    \texttt{pos} \colon \Sigma_{\#}^{*} & \longrightarrow \mathbb{N} 
    \\  
     p & \longrightarrow \max\limits_{i\in \{0,1, \ldots, |p|]} \{i \in \mathbb{N}:~~p_{i} \neq \#\}  \nonumber
\end{align}
\item The $\texttt{sym}(.)$ function:
\begin{align}
    \texttt{sym} \colon \Sigma_{\#}^{*} & \longrightarrow \Sigma \cup <BOS> \\
    p & \longrightarrow \begin{cases}
        <BOS> & \text{if  }\texttt{pos}(p) = 0 \\
        p_{\texttt{pos}(p)} & \text{elsewhere} \nonumber
    \end{cases}
\end{align}
\end{itemize}

 $\bullet$ \textbf{Example:} For the alphabet $\Sigma= \{a,b\}$ and the pattern $p =a\# a\#$. The last position held by a symbol in $\Sigma$ in $p$ is the position $3$. It is held by the symbol $a$. Consequently, for this example, we have $\texttt{pos}(p) = 3$, and $\texttt{sym}(p) = 'a'$. \\
For patterns that contain only the symbol $'\#'$, e.g. $p' = \#\#\#\#$, we have $\texttt{pos}(p') = 0$ and $\texttt{sym}(p') = \text{BOS}$.

Next, we shall see how to reformulate the equation \eqref{computesipw} using the functions $\texttt{pos}(.)$, and $\texttt{sym}(.)$.

For a given pattern $p$ in $\Sigma_{\#}^{*}$, define the language $\Tilde{L}_{p}$ over $\Sigma^{*}$ described as follows:
\begin{equation} \label{tildelp}
\Tilde{L_{p}} \myeq \{ w \in \Sigma^{|p|}:~~w_{\texttt{pos}(p)} = \texttt{sym}(p) \}
\end{equation}

By convention, if $\texttt{pos}(p) = 0$, $\Tilde{L}_{p}$ is equal to $\#^{|p|}$.

Given that $P \in \texttt{MARKOV}$, we have 
\begin{equation} \label{crucialobservation}
P( w'_{j+1} \in L_{p_{j+1}}| w'_{i:j} \in L_{p_{1:j}}) = P( w'_{j+1} \in L_{p_{j+1}} | w'_{1:j} \in \Tilde{L}_{p_{1:j}})
\end{equation}

At this stage, a key observation is that the quantity present in the right-hand side of the equation \eqref{crucialobservation} depends only on $P$, $p_{j+1}$, $\texttt{pos}(p_{1:j}),~\texttt{sym}(p_{1:j})$, and $j$. Indeed, by definition of the language $\Tilde{L}_{p}$ (equation \eqref{tildelp}), the language $\Tilde{L}_{p_{1:j}}$ depends only on these last three parameters. 

To make this dependency appearing explicitly, we shall introduce a definition of a new function  $\mathbf{G}$ given as follows:
\begin{equation} \label{G}
  \mathbf{G}\left(p_{j+1}, \texttt{pos}(p_{i,j}), \texttt{sym}(p_{1:j}), j,P \right) \myeq P( w'_{j+1} \in L_{p_{j+1}} | w' \in \Tilde{L}_{p_{1:j}})
\end{equation}

Using the equality \eqref{crucialobservation}, we can rewrite the constraint \eqref{computesipw} with this newly introduced notation as: 
\begin{small}
\begin{align} \label{yes}
    \forall p \in \Sigma_{\#}^{|w|}:~~f_{A_{w,i,P}} (p) = & \frac{1}{P_{init}(w_{1} \in L_{\texttt{swap}(p,i)_{1}})} \cdot \frac{1}{ \prod\limits_{j=1}^{|w|-1} \mathbf{G} (\texttt{swap}(p,i)_{j+1},  \texttt{pos} (\texttt{swap}(p,i)_{1:j}), \texttt{sym}(\texttt{swap}(p,i)_{1:j}), j, P)} 
\end{align}
\end{small}

Toward the stated objective of constructing a deterministic WA over $\Sigma_{\#}^{*}$ that computes a language satisfying the constraint \eqref{computesipw}, the expression \eqref{yes} offers a better reformulation of this equation by considering two aspects:
\begin{enumerate}
    \item The product terms forming the right-hand side of expression \eqref{yes} offers a compressed representation of the history of the processed pattern required to perform next processing operations, by maintaining only the current position in the sequence, the last symbol different that $\#$ encountered during the forward run and its position in the sequence. 
    \item The functions $\texttt{pos}(.)$ and $\texttt{sym}(.)$ can be easily simulated by a sequential machine that processes sequences from left-to-right, such as WAs. Specifically, for any pattern $p \in \Sigma_{\#}^{*}$ and a symbol $\sigma \in \Sigma$, we have 
    $$\texttt{pos}(p \sigma) = \begin{cases}
        \texttt{pos}(p) & \text{if }  \sigma = \# \\
        \texttt{pos}(p)+1 & \text{elsewhere}
    \end{cases}$$
    $$
      \texttt{sym}(p \sigma) = \begin{cases}
          \texttt{sym}(p) & \text{if  } \sigma = \# \\
          \sigma & \text{elsewhere}
      \end{cases}
    $$
\end{enumerate}
 We shall leverage these two insights to construct a DWA that simulates the computation of the expression \eqref{yes}. 
 
  Assume for now that the function $\mathbf{G}$ can be computed in polynomial time with respect to the input instance (this fact will be proved later in this segment), a polynomial-time construction of  $A_{w,i,P}$ that satisfies the constraint \eqref{twip}:

\begin{itemize}
    \item \textit{The state space:} $Q = \{0,1,\ldots,|w|\}^{2} \times (\Sigma \cup <BOS>)$. \\
    The semantics of the elements of a state $q = (k,l,\sigma) \in Q$ correspond to the current position in the sequence, $\texttt{pos}(.)$ and $\texttt{sym}(.)$, respectively.
    \item \textit{The initial state:} 
    $q_{init} = (0,0,<BOS>)$
    \item \textit{The transition function:} Let $q=(k,l,\sigma)$ be a state in $Q$:
    \begin{enumerate}
        \item \textit{Case 1 ($k = 0$):}
          \begin{itemize}
              \item Case 1.1. ($\sigma' = \#$)
                 $$W((0,0,<BOS>), \# ) = \left((1,0, <BOS>),  \frac{1}{P_{init}(w'_{1 } \in L_{\#})}\right)$$
               \item Case 1.2. ($\sigma' \in \Sigma$)
                $$W((0,0,<BOS>), \sigma') = \left((1,1, \sigma'), \frac{1}{P_{init}(w'_{1 } \in L_{\sigma'})} \right)$$
          \end{itemize} 
        \item \textit{Case 2. ($k = i-1$)} For any $\sigma' \in \Sigma_{\#}$
        $$W((i-1,l,\sigma), \sigma') = \left((i,l,\sigma),  \frac{1}{\mathbf{G}(\#, l, \sigma , k+1,P)}\right)$$
        \item \textit{Case 3 ($k \neq i-1 \land k \in [|w|-1]$):}
        \begin{itemize}
        \item \textit{Case 3.1. ($\sigma' =  \#$)} 
        $$W((k,l, \sigma), \#) =  \left((k+1,l,\sigma),\frac{1}{\mathbf{G}(\#, l, \sigma ,k+1, P)}\right) $$
             
        \item \textit{Case 3.2. ($\sigma' \in \Sigma$)}
        $$W((k,l, \sigma), \sigma') = \left( (k+1,k+1,\sigma')) = \frac{1}{\mathbf{G}(\sigma', l, \sigma ,k+1, P)} \right)$$
        \end{itemize}         
    \end{enumerate} 
    \item \textit{The set of final states:} $F = \{|w|\} \times \{0,1, \ldots, |w|\} \times (\Sigma \cup < \text{BOS}>)$ 
\end{itemize}

$\bullet$ \textbf{Note:} \textit{The case $k=i-1$  in the algorithmic construction outlined above corresponds to the case where the swap operation is taken into account. The adaption of this construction to algorithm $\mathcal{A}_{2}$ consists simply at omitting this case and considering only cases 1 and 3, where case $3$ covers the set $k \in [|w|-1]$. }

For illustrative purposes, we shall give next an example of the path followed by a pattern in the constructed DWA.

$\bullet$ \textbf{Example.} Fix the alphabet $\Sigma = \{a,b\}$ and $P \in \texttt{MARKOV}$. Let $w = aabab$ the instance to explain and the symbol for which we aim at computing the SHAP score is the third symbol, i.e. $i=3$.

Let's consider the pattern $p = \#\#aab$. 
By Bayes' rule, the probability of generating a sequence that follows the pattern $\texttt{swap}(p,3) = \#\#\#ab$ is equal to:
\begin{align*}
P(w' \in L_{\#\#\#ab}) &= P(w'_{5} \in L_{b} | w'_{1:4} \in L_{\#\#\#a} ) \cdot P(w'_{4} = L_{a}| w'_{1,3} \in L_{\#\#\#}) \cdot P(w'_{3} \in L_{\#}|w'_{1,2} \in L_{\#\#}) \\
&~~\cdot P(w'_{2} \in L_{\#} | w'_{1} \in L_{\#}) \cdot P(w'_{1} \in L_{\#})\\
&= \textbf{G}(b, 4, a, 5, P) \cdot \textbf{G}(a, 0, <\text{BOS}>, 4, P) \cdot \textbf{G}(\#, 0, <\text{BOS}>, 3,P) \cdot \textbf{G}(\#,0,<\text{BOS}>,2,P)  \\
&~~ \cdot P_{init}(w'_{1} \in L_{\#})
\end{align*}

 $\textbf{G}(b, 4, a,5,P)$ holds the semantics of the conditional probability of generating the symbol $b$ at position $5$ given that the symbol $a$ is generated at position $4$. Similarly, $\textbf{G}(a, 0, <\text{BOS}>, 4,P)$ holds the semantics of the marginal probability of generating the symbol a at position $4$.

The unique path followed by the pattern $p$ on the DWA constructed above is: 
$$(0,0, <\text{B0S}>)~\#~(1,0,<\text{BOS}>)~\#~(2,0,<BOS>)~a~(3,0,<BOS>)~a~(4,4,a)~b~(5,5,b)$$
 
The weight assigned to this path by the constructed DWA is equal to $$\frac{1}{P_{init}(w'_{1} \in L_{\#})} \cdot \frac{1}{\mathbf{G}(\#,0, < \text{BOS} >, 2, P )} \cdot \frac{1}{\mathbf{G}(\#,0, < \text{BOS} >, 3, P )} \cdot \frac{1}{\textbf{G}(a,0, < \text{BOS} >, 4, P )} \cdot \frac{1}{\mathbf{G}(b,4, a, 5, P )}$$

By noting that for any $\sigma \in \Sigma_{\#}:~~P_{init}(w' \in L_{\#}) = \mathbf{G}(\sigma, 0, < \text{BOS} >, 1, P)$,  the weight of this path is equal to $\frac{1}{P(w' \in L_{\#\#\#ab})} = \frac{1}{P(w' \in L_{\texttt{swap}( \#\#aab,3)})}$.

$\bullet$ \textbf{Computation of the function $\mathbf{G}$.} \\
In order for this constructed DWA to run in time polynomial in the size of its input instance, a necessary and sufficient condition is that the computation of the function $\mathbf{G}$, can also be performed in polynomial time. We shall prove next that this last statement is true.

Formally, the computational problem associated to the function $\mathbf{G}$ is given as follows:

$\bullet$ \textbf{Problem:} The computational problem $\mathbf{G}$ \\
\textbf{Instance:} $\sigma' \in \Sigma_{\#}$, two integers $n,~m > 0$ such that $n < m$, a symbol $\sigma \in \Sigma \cup <\text{BOS}>$ and $P \in \texttt{MARKOV}$.  \\
\textbf{Output:} Compute $G(\sigma',n, \sigma,m,P)$ (equation \eqref{G}).

For an input instance $<\sigma',n,\sigma,m,P>$, the quantity $\mathbf{G}(\sigma', n, \sigma, m,P)$ refers to the conditional probability of generating a symbol in $L_{\sigma'}$ at position $m$ given that the symbol $\sigma$ has been generated at position $n$. In essence, the  computational problem $\mathbf{G}$ is reduced to the classical problem of inference in Bayesian Networks \cite{koller}. In general, the exact inference in Bayesian Networks is intractable \cite{mustapha10}. However, in our case, leveraging the Markovian structure of the probability distribution enables building a tractable solution for the problem using a dynamic programming approach.  

Fix an input instance $<\sigma',n,\sigma,m,P>$ of the problem $\mathbf{G}$. Define the random vector $(X_{n},\ldots, X_{m})$ that takes values over the set $\Sigma^{m-n}$. Its joint probability distribution is given as follows:
$$Q(\sigma_{n}, \ldots, \sigma_{m}) = Q_{init}(\sigma_{n}) \cdot  \prod\limits_{i=n}^{m-1} P_{i}( \sigma_{i+1} | \sigma_{i})$$
such that 
$$Q_{init}(\sigma_{n}) = \begin{cases}
    1 & \text{if } \sigma_{n} = \sigma \\
    0 & \text{elsewhere}
\end{cases}$$
It's easy to observe that 
\begin{equation} \label{GwithQ}
G(\sigma', n, \sigma, m, P) = Q(X_{m} \in L_{\sigma'})
\end{equation}

If $\sigma' = \#$, the computation of $G(\sigma', n, \sigma,m,P)$ is trivial. Indeed, the fact that $L_{\#} = \Sigma$ and $Q(X_{m} \in \Sigma) = 1$ entail, by equation \eqref{GwithQ} that $G(\#,n, \sigma,m,P) = 1$.

For the general case $\sigma' \in \Sigma$, a recursive formula to compute $G(\sigma',n,\sigma,m,P)$ can be obtained, using Bayes' rule, as follows:
\begin{align}
    \mathbf{G}(\sigma' , n,\sigma,m,P) &= Q(X_{m} = \sigma') \\ \nonumber
    &= \sum\limits_{\Tilde{\sigma} \in \Sigma} Q(X_{m-1}=\Tilde{\sigma} \land X_{m} = \sigma) \\ \nonumber
    &= \sum\limits_{\Tilde{\sigma} \in \Sigma} Q(X_{m} = \sigma | X_{m-1} = \Tilde{\sigma}) \cdot Q(X_{m-1} = \Tilde{\sigma}) \\ \nonumber
    &= \sum\limits_{\Tilde{\sigma} \in \Sigma} P_{m-1}(\sigma|\Tilde{\sigma}) \cdot G(\Tilde{\sigma}, n, \sigma, m-1,P ) \label{recursive}
\end{align}

This last equation provides a recursive formula that enables the computation of $\mathbf{G}$ using a dynamic programming approach. The outline of this approach is given as follows:
\begin{itemize}
    \item \textbf{Base case:} $m = n + 1$
    \begin{enumerate}
        \item If $n = 0$:
              $$\mathbf{G}(\sigma', 0, \sigma, m, P) = P_{init}(\sigma')$$
        \item If $n>0$:
    $$\mathbf{G}(\sigma', n, \sigma, m, P) = P_{n+1}(\sigma'|\sigma)$$
    \end{enumerate}
    \item \textbf{General case:} $m > n+1$
    $$ \mathbf{G}(\sigma' , n,\sigma,m,P) = \sum\limits_{\Tilde{\sigma} \in \Sigma} P_{m-1}(\sigma|\Tilde{\sigma}) \cdot G(\sigma', n, \sigma, m-1,P ) $$
\end{itemize}

The complexity of this dynamic programming algorithm is $O(m . |\Sigma|)$.

\section{Proof of lemma \ref{ddnfstowa}} \label{dnftowa}
Lemma \ref{ddnfstowa} states the existence of an algorithm that takes as input a d-DNF $\Phi$, runs in $O(\texttt{poly}(|\Phi|, |\Phi|_{\#}))$, and outputs a WA that implements the language $L_{\Phi}$. Recall that $L_{\Phi}$ is defined as 
$$L_{\Phi} = \{w \in \Sigma^{|\Phi|}: \texttt{SEQ}^{-1}(w)~~\text{satisfies}~~\Phi\}$$

The unweighted language $L_{\Phi}$ includes the set of satisfying variable assignments of the boolean variables arranged in a sequence format.

The structure of the algorithm that performs this task follows two steps:
\begin{enumerate}
  \item Encode every clause $C$ in the input d-DNF in the form of a DFA. The resulting DFA accepts the language $L_{C}$.
  \item Perform a union operation over all these DFAs to obtain a resulting WA. The key observation at the heart of this step is that, for the case of disjoint DNFs the union operation can be performed using a basic sum operation over DFAs constructed in the first step.
\end{enumerate}
Next, we shall provide details of these two steps of the algorithmic construction. 

\subsection{Step 1: Encoding clauses as DFAs.}
 The basic intuition for performing this step is  that an equivalent representation of the language accepted by a clause can be alternatively represented  by a pattern of length $|p|$. On the other hand, a pattern of length $|p|$ can be implemented using a DFA of size at most $|p| + 1$. 

Let $C = l_{1} \land \ldots \land l_{k}$ be a conjunctive clause over $N$ boolean variables. We shall denote by $L_{C}$ the set of satisfying variable assignments of the clause $C$ arranged in a sequence format.

The construction of a pattern $p$ such that $L_{p} = L_{C}$ can be performed by scanning the literals of the clause $C$ from left-to-right. Assume that the clause $C$ doesn't possess a variable and its negation in its set of literals \footnote{Clauses that exhibit this degenerate case can be checked and removed before running the algorithmic schema outlined here.}. The algorithmic schema is given as follows:  

$1$ Initialize a pattern $p$ as $\#^{N}$ \\
$2$ For each literal $l_{i}$ in $C$: \\
$~~~~~~~$- If $l_{i}$ corresponds to a variable $X_{k}$, then set $p_{k} = 1$ \\
$~~~~~~~$ - If $l_{i}$ corresponds to the negation a variable $\bar{X}_{k}$, then set $p_{k} = 0$ \\

The algorithmic schema ensures that the language of the outputted pattern accepts all and only sequences that satisfy the constraints enforced by all literals of the clause.  

The pattern construction of a clause in a d-DNF $\Phi$ as well as its conversion to a DFA can be performed in $O(|\Phi|)$ time. And, the size of  resulting DFA is $O(|\Phi|)$. Repeating the same operation over all clauses of $\Phi$ runs in $O(|\Phi| \cdot |\Phi|_{\#})$ time. 

$\bullet$ \textbf{Example:} The pattern associated to $C = X_{2} \land \bar{X}_{4} \land \bar{X}_{3}$ over the set of boolean variables $\{X_{1}, X_{2}, X_{3}, X_{4}, X_{5}\}$ is $\#101\#$.

\subsection{Step 2: The union of DFAs representing clauses.}
Fix a d-DNF $\Phi = C_{1} \lor \ldots \lor C_{M}$ over $N$ boolean variables. Let $A_{1},~\ldots,~A_{M}$ be a collection of DFAs (outputted by step 1) that accept the languages $L_{C_{1}},~\ldots,~L_{C_{M}}$, respectively. The main problem of this step is how to exploit DFAs outputted in the first step to construct a WA $A$ such that $I_{L_{\Phi}} = f_{A}$. 

The main intuition at this point is to note that for a d-DNF, $I_{L_{\Phi}}$ can be expressed as as a sum of the indicator functions of $\{I_{L_{C_{i}}}\}_{i \in [M]}$. Since $f_{A_{i}} = I_{L_{C_{i}}}$ for any $i \in [M]$, then $f_{A}$ can be computed as a sum over languages computed by DFAs . This observation will result into a reduction of the problem of constructing  a WA $A$ that computes $I_{L_{\Phi}}$ into performing a sum operation over a collection of DFAs. Fortunately, WAs are closed under the sum operation. Moreover, it can be computed in polynomial time with respect to the size. 
\begin{lemma} \label{inclusionexclusion}
    Let $\Phi = C_{1} \lor \ldots \lor C_{M}$ be a disjoint DNF over $N$ boolean variables. We have:
    $$I_{L_{\Phi}} = \sum\limits_{i \in [M]}I_{L_{C_{i}}}$$
\end{lemma}
\begin{proof}
    Let  $\Phi = C_{1} \lor \ldots \lor C_{M}$ be a disjoint DNF over $N$ boolean variables. Let $w$ be an arbitrary sequence in $\{0,1\}^{n}$. Our claim is that $I_{L_{\Phi}}(w) = \sum\limits_{i \in [M]} I_{L_{C_{i}}}(w)$. Note that $L_{\Phi} = \bigcup\limits_{i=1}^{M} L_{C_{i}}$ by definition of $\Phi$. Also, $\bigcap\limits_{i=1}^{M} C_{i} = \emptyset$ by the disjoint property of $\Phi$.
    \begin{itemize}
        \item \textit{Case 1 ($w \notin L_{\Phi}$):} This implies that $I_{L_{\Phi}}(w) = 0$.  On the other hand, $w \notin L_{\Phi}$ and  $L_{\Phi} = \bigcup\limits_{i=1}^{M} L_{C_{i}}$ implies that 
        $$\forall i \in [M]: w \notin L_{C_{i}} \implies \forall i \in [M]: I_{L_{C_{i}}}(w) = 0 \implies \sum\limits_{i=1}^{M} I_{L_{C_{i}}}(w)=0$$
        \item \textit{Case 2 ($w \in L_{\Phi}$):} In this case, $I_{L_{\Phi}}(w) = 1$.
        On the other hand, $L_{\Phi} = \bigcup\limits_{i=1}^{M} L_{C_{i}}$ implies that there exists at least one clause $C_{i}$ such that $w \in L_{C_{i}}$. This fact combined with the fact that  $\bigcap\limits_{i=1}^{M} C_{i} = \emptyset$ implies that this clause is unique. Denote by $C^{*}$ this clause. We have $I_{L_{C^{*}}}(w) = 1$. And, $I_{L_{C}}(w) = 0$ for any clause $C \in \{C_{i}\}_{i \in [M]} \setminus C^{*}$. Consequently, $\sum\limits_{i=1}^{M} I_{L_{C_{i}}}(w) = 1$.
    \end{itemize}
\end{proof}
The result of lemma \ref{inclusionexclusion} implies that a WA $A$ that computes the language $I_{L_{\phi}}$ satisfies: 
\begin{equation} \label{fA}
f_{A} = \sum\limits_{i=1}^{M} I_{L_{C_{i}}} = \sum\limits_{i=1}^{M} f_{A_{i}}
\end{equation}

For two WAs $A_{1}=<\alpha, \{A_{\sigma}\}_{\sigma \in \Sigma} \beta>$ and $A_{2} = <\alpha', \{A'_{\sigma}\}_{\sigma \in \Sigma}, \beta'>$, the  WA whose set of parameters is given as 
$$<\begin{pmatrix} \alpha \\\alpha' \end{pmatrix} , \{\begin{pmatrix}
    A_{\sigma} & \mathbf{O}_{\texttt{size}(A_{1}) \times \texttt{size}(A_{2})} \\
    \mathbf{O}_{\texttt{size}(A_{2}) \times \texttt{size}(A_{1})} & A'_{\sigma}
\end{pmatrix}\}_{\sigma \in \Sigma}, \begin{pmatrix}
    \beta \\ 
    \beta'
\end{pmatrix}>$$
where $\mathbf{0}_{n \times m}$ is the zero matrix in $\mathbb{R}^{n \times m}$,
computes the language $f_{A} + f_{A'}$. 

The resulting WA runs in $O(\texttt{size}(A_{1}) + \texttt{size}(A_{2}))$ time, and has size equal to $\texttt{size}(A_{1}) + \texttt{size}(A_{2})$.

Hence, the construction of the target WA $A$ by  performing the sum operation over $\{A_{i}\}_{i \in [M]}$ as outlined by equation \eqref{fA} would take $O(\sum\limits_{i=1}^{M} \texttt{size}(A_{i}))$ operations. Since the DFAs $\{A_{i}\}_{i \in [M]}$ have size equal to $O(|\Phi|)$. Then, the overall operation runs in $O(|\Phi| \cdot |\Phi|_{\#})$.

\end{document}